\documentclass[10pt,journal,compsoc]{IEEEtran}
\ifCLASSOPTIONcompsoc
  \usepackage[nocompress]{cite}
\else
  \usepackage{cite}
\fi
\ifCLASSINFOpdf
\else
\fi
\usepackage{graphicx}
\usepackage{amsmath}
\usepackage{amssymb}
\usepackage{amsthm}
\usepackage{subcaption}
\usepackage{blkarray}
\usepackage{threeparttable}
\usepackage{multirow}
\usepackage{makecell}
\usepackage{colortbl}
\usepackage{array}
\usepackage{diagbox}
\usepackage{booktabs}
\usepackage{tikz}
\usepackage{fbox}
\usepackage{bm}
\usepackage{bbm}
\usepackage{url}
\usepackage[linesnumbered,ruled,boxed,norelsize,vlined,commentsnumbered]{algorithm2e}
\usepackage{comment}
\usepackage{marginnote}
\usepackage{sidecap}
\usepackage{adjustbox}
\usepackage{float}
\usepackage{bbding}
\usepackage{adjustbox}

\definecolor{cvprblue}{rgb}{0,0,1}%{0.21,0.49,0.74}
\usepackage[pagebackref,breaklinks,colorlinks,citecolor=cvprblue]{hyperref}

\usepackage[capitalize]{cleveref}
\crefname{section}{Sec.}{Secs.}
\Crefname{section}{Section}{Sections}
\Crefname{table}{Table}{Tables}
\crefname{table}{Tab.}{Tabs.}

\newtheorem{lemma}{Lemma}
\newtheorem{definition}{Definition}
\newtheorem{proposition}{Proposition}

\newcommand{\ie}{{\textit{i.e.}}}
\newcommand{\etal}{{\textit{et al.}}}
\newcommand{\eg}{{\textit{e.g.}}}

\newcommand{\revise}[1]{{\color{black}{#1}}}
\newcommand{\qa}[3]{\ifmmode{\color{blue}#3}\else\marginpar{\centering\color{blue}\textbf{R#1 A#2} }{\color{blue}#3}\fi}
\graphicspath{{./Fig/}}

\begin{document}
%\title{An Unsupervised Clustering Method for Superquadric Fitting}

% \title{Superquadric Fitting Using Unsupervised Clustering Analysis}

%\title{An Unsupervised Clustering-Inspired Method for Superquadric Fitting}

\title{A Holistic Method for Superquadric Fitting Using Unsupervised Clustering Analysis}

%\title{Superquadric Fitting Using Unsupervised Clustering Analysis}

	\author{Mingyang~Zhao,~
		Sipu Ruan,~
		 Xiaohong Jia%,~ 
		%  %Shiqing Xin, 
		% Dong-Ming Yan
		\IEEEcompsocitemizethanks{
			\IEEEcompsocthanksitem M. Zhao and X. Jia are with the State Key Laboratory of Mathematical Sciences, Academy of Mathematics and Systems Science, Chinese Academy of Sciences, and also with the University of Chinese Academy of Sciences. E-mail: \{zhaomingyang, xhjia\}@amss.ac.cn.\protect
			\IEEEcompsocthanksitem S. Ruan  is with the Robotics Institute, School of Mechanical Engineering and Automation, Beihang University. E-mail:ruansp@buaa.edu.cn.\protect
			%\IEEEcompsocthanksitem X. Li is with the Beijing Huairou Hospital of Traditional Chinese Medicine, Beijing, China. E-mail: xuel46690@gmail.com\protect
			%\IEEEcompsocthanksitem L. Ma is with the Peking University, Beijing,
			%China. E-mail: lei.ma@pku.edu.cn\protect
			% \IEEEcompsocthanksitem D. Yan is with the Institute of Automation, Chinese Academy of Sciences. E-mail: yandongming@ia.ac.cn.

            \IEEEcompsocthanksitem Mingyang Zhao is the corresponding author.
		}% <-this % stops an unwanted space
		%\thanks{Manuscript}
		}

\markboth{IEEE TRANSACTIONS ON PATTERN ANALYSIS AND MACHINE INTELLIGENCE}{Shell \MakeLowercase{\textit{et al.}}: An Unsupervised Clustering Method for Superquadric Fitting}

\IEEEtitleabstractindextext{%
\begin{abstract}
This work presents a novel method for fitting superquadrics to point clouds under the contamination of noise and outliers, which has many applications for shape modeling across diverse fields. Unlike prior approaches that either exclusively focus on fitting rigid or deformable superquadrics, or suffer from robustness and numerical instability issues, our method redefines the problem from a new \emph{unsupervised clustering perspective}, enabling the \emph{holistic} fitting of both rigid and deformable superquadrics within a unified framework. Central to our approach is a stable optimization function inspired by unsupervised clustering analysis, where we formulate the point cloud data and samples from the potential parametric surface as \emph{clustering members} and \emph{centroids}, respectively. Then, the clustering process with dynamic updates to centroid locations serves as a direct proxy for optimizing superquadric parameters, establishing a principled link between geometric fitting and clustering dynamics. We further derive the relationship between pairwise computations of clustering centroids and clustering members to orthogonal distances, effectively eliminating the need for the time-consuming surface sampling process. Moreover, our formulation provides \emph{closed-form analytical solutions} for both the fuzzy membership degree vector and the covariance matrix, ensuring efficient iteration optimization and enabling more effective handling of geometric deformations. In addition, we provide a \emph{theoretical certificate} of convergence analysis and demonstrate that the clustering-inspired fitting method can \emph{escape local minima} by inherently increasing the convexity of the objective function. We experimentally demonstrate the improvements of our superquadric fitting method in accuracy, robustness, and stability over state-of-the-art approaches. 
We also show that our method effectively avoids convergence to local minima, yielding smaller point-to-surface distances, particularly for the highly tapered shapes. Additionally, we illustrate the versatility of our method in diverse applications, including \emph{\revise{complex object approximation}} via multiple superquadric representation, \emph{type-specific primitive fitting}, \emph{geometry editing}, and \emph{medical modeling}. \revise{The implementation is publicly available at \url{https://github.com/zikai1/SuperquadricFitting}.}

\end{abstract}

\begin{IEEEkeywords}
Superquadric fitting, parameter estimate, tapering deformation, bending deformation, unsupervised clustering.
\end{IEEEkeywords}}

\maketitle

\IEEEdisplaynontitleabstractindextext

% \fi
\IEEEpeerreviewmaketitle

\IEEEraisesectionheading{\section{Introduction}\label{sec:introduction}}
\IEEEPARstart{F}{itting} geometric primitives to 3D point clouds serves as a critical bridge between low-level scattered data and high-level \emph{structural representations} of underlying 3D shapes~\cite{li2019supervised}. {A}{s} a powerful modeling primitive, \emph{superquadratic surfaces} (\emph{superquadrics}) extend traditional CAD models (\eg, cylinders, spheres) with great \emph{flexibility} and \emph{versatility}~\cite{gross1988error}. These surfaces encompass a rich shape vocabulary, ranging from cuboids, ellipsoids, and octahedra to their intermediate forms (the left panel of \cref{fig:teaser}), encoded with only five parameters in canonical form. Additionally, their modeling power can be further expanded via global deformations, such as tapering or bending, enabling the description of more complex geometries that transcend standard primitives (the right panel of \cref{fig:teaser}).

Benefiting from their \emph{low-dimensional} and \emph{compact} characterization as well as  \emph{sufficient} shape representation capacity, superquadrics have found widespread applications across various domains, including computer graphics~\cite{gielis2003superquadrics,schultz2010superquadric,choi20243doodle}, robotics~\cite{ruan2022efficient,han2023sq, xing2024object,zhao2024bayesian}, and computer vision~\cite{paschalidou2019superquadrics,liu2022robust,li2024shared}. For example, Choi \etal~\cite{choi20243doodle} leveraged superquadric contours to generate smooth volumetric outlines across different viewpoints for 3D stroke drawing. Xing \etal~\cite{xing2024object} pioneered the integration of superquadrics into object-based SLAM systems, facilitating robust camera localization.

{The} problem of {single} superquadric fitting necessitates optimizing a set of unknown parameters to define a superquadric surface that accurately models a given shape or its partial geometry, such as those described by 3D point cloud samples. This parameter estimation process is commonly cast as a least-squares problem, achieving satisfactory local minima in conventional scenarios. However, most existing approaches for superquadric fitting predominantly concentrate on rigid geometry, thereby struggling to accurately capture \emph{spatial deformations} (\eg, the tapering and bending deformations illustrated in the right panel of \cref{fig:teaser}). Such geometric deformations pose a more challenging optimization problem yet are ubiquitously occurred in both man-made and natural environments. Moreover, these methods often insufficiently tackle the impact of noise, outliers, and partial data during algorithmic design, which inevitably compromises the fitting results in terms of both accuracy and applicability.

{To address these issues, this work focuses on the parameter estimation and geometric optimization of a single superquadric and introduces a novel method for fitting superquadric surfaces to point cloud data.} The proposed approach accommodates both rigid and deformable geometries within a \emph{holistic} optimization framework. We introduce a new clustering-based solution to the fitting problem, formulating it as an \emph{unsupervised clustering process}. Specifically, given a set of unorganized point samples, we designate them as \emph{clustering members} and the potential superquadric as the continuous shape space where \emph{clustering centroids} reside. Then, analogous to clustering, the dynamic update of clustering centroids corresponds to the parameter optimization of superquadrics.  
To avoid inefficient and nonuniform sampling from potential surfaces, we further derive a relationship that correlates the pairwise distance computation between clustering centroids and members with the orthogonal distance, and thus significantly accelerates the {parameter} optimization process. We highlight the advantages of our clustering-inspired optimization function from a \emph{convex optimization perspective}, demonstrating its ability to automatically escape local minima and thereby achieve a \emph{tighter} shape approximation. Additionally, we provide \emph{closed-form solutions} for both the fuzzy membership degree vector and the fuzzy covariance matrix during each iteration, which enables fast and efficient implementations. 
The proposed method enables \revise{joint} optimization of \emph{shape parameters}, \emph{spatial transformations}, and \emph{geometric deformations}, and is guaranteed to converge.

\begin{figure}[t]
    \centering
\includegraphics[width=\linewidth]{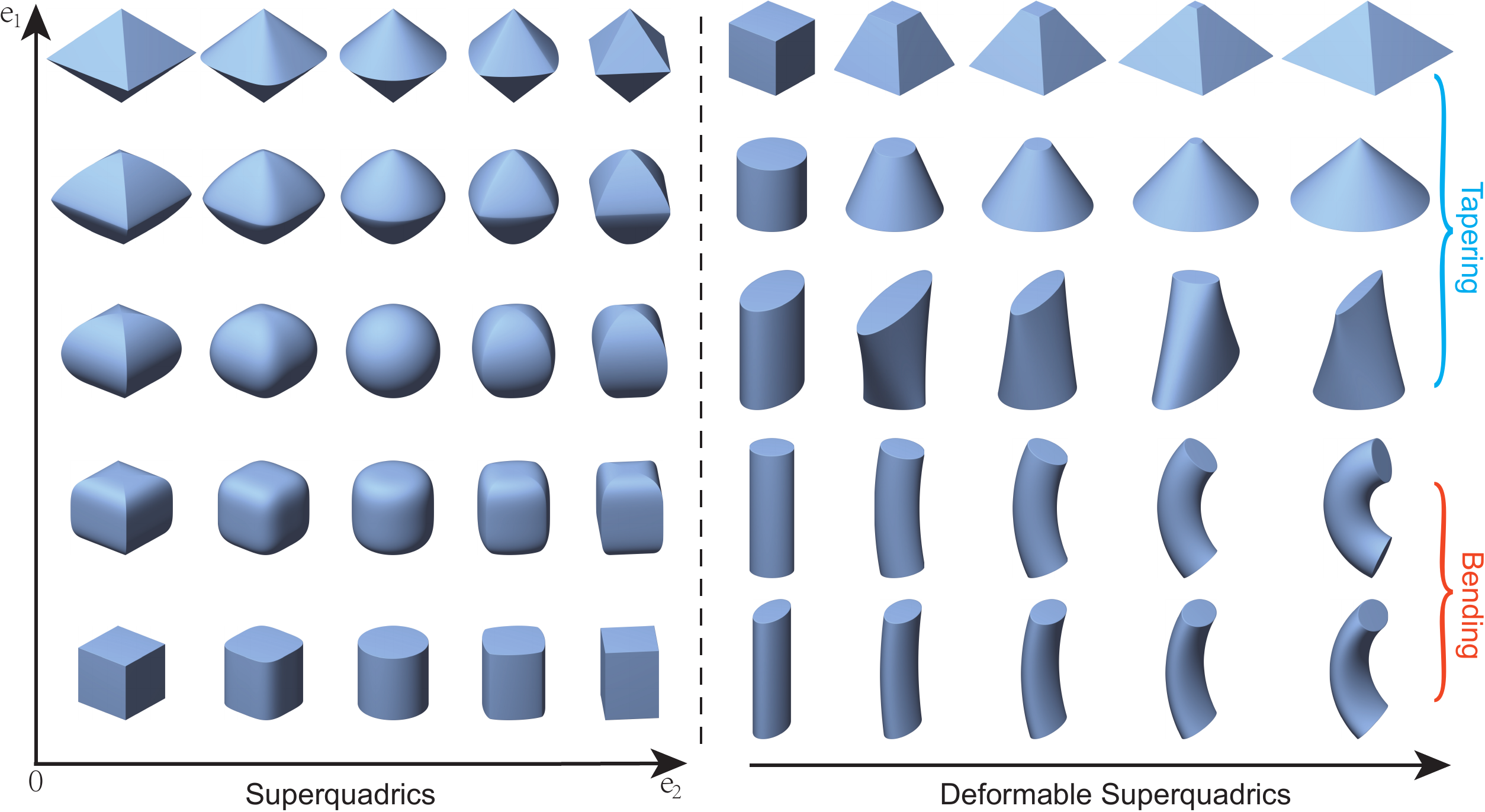}
\caption{Gallery of superquadrics. Left: Rigid superquadrics. Right: Deformable superquadrics. Our method enables holistic fitting of both categories within a single unsupervised clustering optimization framework.}
\label{fig:teaser}
\end{figure}

We conduct extensive experiments to validate the superior accuracy, robustness, and stability of our superquadric fitting method under challenging conditions such as noise, outliers, partial data, and geometric deformations, as well as its broad applicability to real-world tasks including shape \revise{approximation} and medical modeling. We also empirically investigate the convergence behavior of the designed algorithm and its capacity to escape local minima. Experimental results demonstrate a significant performance margin over baseline approaches, confirming the method's \revise{stable} convergence and exceptional capacity to escape local minima when handling severe geometric deformations and external disturbances.

Our contributions can be summarized as follows:
\begin{itemize}
 
\item We present a novel and unified framework for {single} superquadric fitting inspired by unsupervised clustering principles. The method demonstrates superior performance in modeling both rigid and deformable geometries and is guaranteed to converge.

\item We derive and establish a theoretical relationship between the pairwise distance computation of clustering centroids and their members using orthogonal distance, eliminating the need for inefficient and nonuniform surface sampling.

\item {Our method provides closed-form solutions for the fuzzy membership degree and the fuzzy covariance matrix at each iteration, incorporates an entropy regularization term that inherently avoids local minima, and significantly improves robustness to noise, outliers, and geometric deformations.}
\end{itemize}

The rest of this paper is structured as follows. \cref{sec:related_work} reviews the state-of-the-art approaches closely related to ours, followed by \cref{sec:preliminary}, which introduces the preliminaries of superquadric definitions and fuzzy clustering analysis. The proposed method is detailed in \cref{sec:method}, where we show how to perform superquadric fitting under an unsupervised clustering framework. 
\cref{sec:experiments} presents experimental validations, demonstrating the method's accuracy, robustness, stability, and applicability across diverse scenarios. 
We discuss the limitations of the current work in~\cref{sec:limitation}.~\cref{sec:conclusions} concludes this study and outlines the future directions.

\section{Related Work}\label{sec:related_work}
In this section, we review superquadric fitting algorithms related to our work. 
Since the introduction of superquadrics by~\cite{barr1981superquadrics}, significant efforts have been made to improve the fitting accuracy of superquadrics for shape representation and modeling by revising or redesigning new loss functions. For instance,~\cite{solina1987three,solina1990recovery} defined a loss function based on the \emph{implicit inside-outside equation} (\emph{i.e.}, algebraic distance) of superquadrics.~\cite{gross1988error} investigated the bias of the implicit function and proposed a more stable radial Euclidean distance as a modification.~\cite{zhang2003experimental}
compared the implicit and radial objective functions and provided a guidance for selecting the optimal objective function in superquadric representation tasks.~\cite{hu1995robust} introduced a weighted energy function to robustly extract superquadrics from noisy data.~\cite{van1998fitting} improved Solina's algorithm by modifying the distance measure during iteration and incorporating a background constraint.

However, as noted in~\cite{vaskevicius2017revisiting}, all the aforementioned approaches is prone to numerical instabilities in certain regions of the parameter space (\emph{i.e.}, when either of the shape parameters of superquadrics approaches zero). To mitigate this issue,~\cite{vaskevicius2017revisiting} revisited the previous superquadric fitting problem and proposed a numerically stable formulation for evaluating the surface function and its gradient using a new  auxiliary function. More recently,~\cite{liu2022robust} formulated superquadric fitting as a maximum likelihood estimation problem, demonstrating impressive performance across a series of tests. Nevertheless, both of these recent methods focus solely on rigid superquadrics, leaving open questions about algorithm capabilities for more complex deformable scenarios.

Learning-based methods have also been explored to address the superquadric fitting problem~\cite{vsircelj2020segmentation}. For instance,~\cite{oblak2020learning} designed a learnable objective based on the superquadric inside-outside function and developed two learning strategies for training convolutional neural networks to predict superquadric parameters. Similarly, DSQNet~\cite{kim2022dsqnet} adopted a two-stage supervised learning framework: first identifying object shapes via superquadrics, then guiding robotic grasping tasks. {In contrast to most learning-based approaches, which demand extensive data and substantial training time, our method adopts a per-instance \emph{geometric optimization} perspective, enabling direct and fast superquadric fitting.}

Beyond superquadric fitting, some approaches integrate segmentation techniques with parameter optimization for multi-superquadric recovery, enabling part-based shape decomposition~\cite{leonardis1997superquadrics, alaniz2023iterative}, while others directly parse objects for shape abstraction~\cite{hachiuma2019volumetric, paschalidou2019superquadrics, wu2022primitive}. 
While our method focuses on estimating the parameters of a \revise{single} superquadric, including shape, spatial pose, and geometric deformations, it can be seamlessly integrated with these techniques to describe shapes using multiple superquadrics. Moreover, unlike most existing multi-superquadric  approaches that target rigid shapes, our framework naturally incorporates geometric deformations into the recovery process, further extending its applicability to deformable structures. We demonstrate this versatile integration in the following experiments.

\section{Preliminaries}\label{sec:preliminary}
In this section, we begin with the introduction of 
notations used throughout the paper for clarity, followed by the mathematical
preliminaries of rigid and deformable superquadrics, as well as the fuzzy clustering framework. 
\subsection{Notations}

Let $\mathbb{Z}_{+}$, $\mathbb{R}_{+}$, $\mathbb{R}_{\geq0}$, $\mathbb{R}^{n}$, $\mathbb{R}^{m\times n}$, $\mathbb{S}_{++}^{n}$ denote the set of positive integers, positive scalars, non-negative scalars, $n$-dimensional vectors, size $m\times n$ matrices, and size $m\times n$ positive definite matrices, respectively. We adopt the following notation conventions: Lowercase letters for scalars (\eg, $x$), boldface lowercase letters for vectors (\eg, $\mathbf{x}$), and uppercase letters (\eg, \(\mathbf{A}\)) for matrices. For a matrix \(\mathbf{A} \in \mathbb{R}^{m \times m}\), \(\text{Tr}(\mathbf{A})\) denotes its trace, and \(|\mathbf{A}|\) denotes its determinant. The transpose of a vector or matrix is denoted by \((\cdot)^{\top}\). Let \([m]\) represent the set \(\{1, \dots, m\}\). The \(\ell_p\) norm of a vector is denoted by\(\|\cdot\|_p\).

\subsection{Superquadrics}\label{subsec:SQ}
\subsubsection{{Rigid  Superquadrics}} 
Superquadrics~\cite{barr1981superquadrics} are a family of compact geometric primitives that extend the modeling capabilities of quadric surfaces and solids, thereby enabling the powerful representation of a broader range of shapes. Mathematically, a rigid superquadric in {canonical form} can be defined by the following parametric equation:
\begin{equation}
\begin{array}{l}
\bm{x}(\eta, \omega)\!=\!\left[\begin{array}{c}
C_{\eta}^{\varepsilon_{1}} \\
a_{z} S_{\eta}^{\varepsilon_{1}}
\end{array}\right] \otimes\left[\begin{array}{c}
a_{x} C_{\omega}^{\varepsilon_{2}} \\
a_{y} S_{\omega}^{\varepsilon_{2}}
\end{array}\right]\!=\!\left[\begin{array}{c}
a_{x} C_{\eta}^{\varepsilon_{1}} C_{\omega}^{\varepsilon_{2}} \\
a_{y} C_{\eta}^{\varepsilon_{1}} S_{\omega}^{\varepsilon_{2}} \\
a_{z} S_{\eta}^{\varepsilon_{1}}
\end{array}\right], 
\end{array}
\label{eq:parametric}
\end{equation}where $ C_{\alpha}^{\varepsilon}=\cos^{\varepsilon}(\alpha)$, $S_{\alpha}^{\varepsilon} = \sin^{\varepsilon}(\alpha)$, and \(\eta \in [-\pi/2, \pi/2]\), \(\omega \in (-\pi, \pi]\) are parametric angles. The operation \(\otimes\) is the spherical product (\ie, cross product in spherical polar coordinates). The shape parameters \(\varepsilon_{1}\) and \(\varepsilon_{2} \in \mathbb{R}_{+}\) control the superquadric's geometry (squareness) along the {latitude} and {longitude} directions, respectively. Parameters \(a_x\), \(a_y\), \(a_z \in \mathbb{R}_{+}\) define the superquadric size along the \(x\), \(y\), \(z\) coordinates. We focus on the surface fitting of convex superquadrics (\ie, \(\varepsilon_{1}, \varepsilon_{2} \in (0, 2]\)) for 3D point clouds.

Alternatively, we can convert~\cref{eq:parametric} to an \emph{implicit function} to readily determine the relative position of a point $\bm{x} = [x, y, z]\in\mathbb{R}^3$ and a superquadric (\ie, the inside-outside function):
\begin{equation}
F(x, y, z)\!=\!\left(\left(\frac{x}{a_{x}}\right)^{\frac{2} {\epsilon_{2}}}+\left(\frac{y}{a_{y}}\right)^{\frac{2}{ \epsilon_{2}}}\right)^{\frac{\epsilon_{2}}{\epsilon_{1}}}
\!+\!\left(\frac{z}{a_{z}}\right)^{\frac{2}{\epsilon_{1}}}.
\label{eq:implicit}
\end{equation} A point \(\bm{x}\in\mathbb{R}^3\) lies on the surface if \(F(\bm{x}) = 1\), inside if \(F(\bm{x}) < 1\), and outside if \(F(\bm{x}) > 1\). Eqs.~\eqref{eq:parametric} and~\eqref{eq:implicit} describe the superquadrics in a \emph{canonical form}. We can further parameterize superquadrics in a \emph{general position} by applying a rigid transformation \(\mathcal{T}\), \ie, $\mathbf{x}=\mathcal{T}(\mathbf{x})$, where \(\mathcal{T}\) belongs to the \emph{special Euclidean group} \(SE(3)\):
\begin{equation*}
SE(3)\! =\! \left\{\left.\mathcal{T} \!= \!\begin{bmatrix}
\mathbf{R} & \mathbf{t} \\
\mathbf{0}^{\top} & 1
\end{bmatrix} \in \mathbb{R}^{4 \times 4} \right\rvert\, \mathbf{R} \in SO(3), \mathbf{t} \in \mathbb{R}^3\right\}.
\end{equation*}Therefore, fitting rigid superquadrics amounts to optimizing an unknown variable set \(\mathbf{\Theta} = \{\varepsilon_1, \varepsilon_2, a_x, a_y, a_z, \mathcal{T}\}\) that simultaneously describes the superquadric's geometric shape (\ie, $\varepsilon_1, \varepsilon_2, a_x, a_y, a_z$) and spatial pose (\ie, $\mathbf{R}$, $\mathbf{t}$). The left panel of~\cref{fig:teaser} illustrates the shape vocabulary of rigid superquadrics, showcasing canonical forms under different \(\varepsilon_1, \varepsilon_2\) combinations.

\subsubsection{Deformable Superquadrics} 
Although rigid superquadrics can represent a diverse set of shapes in our surroundings, their shape modeling capabilities can be further enhanced by \emph{deforming} them in various ways, such as \emph{tapering} and \emph{bending}~\cite{barr1984global} as representative cases. As illustrated on the right panel of~\cref{fig:teaser}, these deformations allow us to represent more complex shapes observed in both natural and man-made environments.

Specifically, a tapering deformation along the $z$-axis is defined as 
\begin{equation}
\Tilde{{x}}=f_{k_{x}}(z){x},\quad \Tilde{{y}}=f_{k_{y}}(z){y},\quad \Tilde{{z}}={z},
\end{equation}where $f_{k_{x}}$ and $f_{k_{y}}$ are the \emph{linear tapering functions} defined as 
\begin{equation}
f_{k_{x}}(z)=\frac{{k_{x}}}{a_z}z+1,\quad f_{k_{y}}(z)=\frac{{k_{y}}}{a_z}z+1.
\end{equation} Here, ${k_{x}}$, ${k_{y}}\in[-1, 1]$ are the tapering parameters in the $x$- and $y$-axes of the object-centered coordinate system, respectively. Note that the tapered superquadrics will degenerate to rigid ones when $k_x=k_y=0$. Thus, the fitting of tapered superquadrics in a general spatial pose is equivalent to optimizing the parameter set \(\mathbf{\Theta}_{\text{taper}} = \{\varepsilon_1, \varepsilon_2, a_x, a_y, a_z,\) \(\mathcal{T}, k_x, k_y\}\). The top right part of~\cref{fig:teaser} presents several deformable superquadrics subjected to tapering transformations.

Similarly, for bending superquadrics, we adopt a transformation that bends the 
$z$-axis into a circular section, defined by
\begin{equation}
\left\{
\begin{array}{l}
{x}'=x+\cos(\alpha)(R-r),\\
{y}'=y+\sin(\alpha)(R-r),\\
{z}'=\sin(\gamma)(\kappa^{-1}-r), 
\end{array}
\right.
\end{equation} where 
\begin{equation}
\left\{
\begin{array}{l}
r=\cos(\alpha-\arctan(y/x))\sqrt{x^2+y^2},\\
\gamma=z\kappa,\\
R=\kappa^{-1}-\cos(\gamma)(\kappa^{-1}-r).\\
\end{array}
\right.
\end{equation} The parameters \(\kappa \in \mathbb{R}_{+}\) and \(\alpha\in[0, \pi/2]\) denote the \emph{bending curvature} and \emph{bending direction} in the \(x\)-\(y\) plane, respectively. These variables are optimized alongside the rigid superquadric parameters \(\mathbf{\Theta}\), resulting in the extended parameter set \(\mathbf{\Theta}_{\text{bend}} = \{\mathbf{\Theta}, \kappa, \alpha\}\). Interested readers are recommended to refer to~\cite{barr1984global} for detailed derivations.

\subsection{Fuzzy Clustering Analysis}
As a representative unsupervised learning paradigm, fuzzy clustering enables \emph{soft data partitioning} by assigning each point a degree of membership to multiple clustering centroids. This property makes it well-suited for superquadric fitting, where geometric ambiguities (\eg, measurement noise, partial point clouds) necessitate handling of \emph{uncertain data}.

\begin{definition}
Let $\mathbb{R}^n$ denote the data space with data points $\mathbf{X}=\{\bm{x}_i\in \mathbb{R}^n\}_{i=1}^M$. Fuzzy clustering analysis aims to solve the following problem: 
\begin{equation}
\min_{\mathbf{U},\mathbf{V}}\sum_{j=1}^C\sum_{i=1}^M(u_{ij})^r||\bm{x}_i-\bm{v}_j||_2^2, s.t.\sum_{j=1}^Cu_{ij}=1, u_{ij}\in[0,1]. 
\label{eq:fcm_origin}
\end{equation}Here, $\mathbf{U}=[u_{ij}]_{M\times C}\in\mathbb{R}^{M\times C}$ is the {fuzzy membership degree matrix}, $\mathbf{V}=\{\bm{v}_j\in \mathbb{R}^n\}_{j=1}^C$ is the set of clustering centroids consisting of $C\in\mathbb{Z}_{+}$ classes, and $r\in(1,+\infty)$ is the {fuzzy factor}, which controls the clustering fuzziness. 
\end{definition}The closed-form solution for $\mathbf{U}$ in~\cref{eq:fcm_origin} (derived in the \emph{Supplementary Material}) can be expressed as 
\begin{equation}
u_{ij}=\frac{1}{\sum_{k=1}^{C}\left(\frac{\left\|\boldsymbol{x}_i-\boldsymbol{v}_j\right\|^2_2}{\left\|\boldsymbol{x}_i-\boldsymbol{v}_k\right\|^2_2}\right)^{\frac{1}{r-1}}}.
\label{eq:U_origin}
\end{equation} As Euclidean distance-based
clustering algorithms are primarily suitable for spherical
data, Gustafson and Kessel~\cite{gustafson1979fuzzy} further introduced the \emph{Mahalanobis distance} to generalize fuzzy clustering analysis to accommodate ellipsoidal structures. Recently, Chen \etal~\cite{chen2023ell} combined the merits of previous fuzzy clustering analysis approaches and introduced a novel clustering framework based on the robust $\ell_{2,p}$ norm. This framework achieves appealing results on a set of unsupervised clustering analysis tasks: 
\begin{equation}
\begin{small}
\begin{gathered}
\min_{\mathbf{U},\mathbf{V},\mathbf{\Sigma},\mathbf{\zeta}}\sum_{i, j=1}^{M, C}{u}_{ij}||\mathbf{\Sigma}_j^{-\frac{1}{2}}(\bm{x}_i\!-\!\bm{v}_j)||_2^p\!+\!{u}_{ij}\text{log}|\mathbf{\Sigma}_j|\!+\!\lambda {u}_{ij}\text{log}\frac{{u}_{ij}}{\zeta_j},\\
s.t.\quad\sum_{j=1}^C{u}_{ij}=1,\quad\sum_{j=1}^C\zeta_{j}=1,\quad u_{ij}, \zeta_{j}\in(0,1).
\end{gathered}
\label{eq:loss_cluster}
\end{small}
\end{equation}where $\mathbf{\Sigma}_j\in \mathbb{S}^{n}_{++}\triangleq\{\mathbf{A}\in\mathbb{R}^{n\times n}|\bm{x}^{T}\mathbf{A}\bm{x}>0, \forall \bm{x}\in\mathbb{R}^n\}$ denotes the fuzzy covariance matrix of the $j$-th class, with the corresponding determinant equivalent to $|\mathbf{\Sigma}_j|\in\mathbb{R}_{+}$. $\lambda\in\mathbb{R}_{+}$ is a regularization parameter. $\bm{\zeta}=[\zeta_1, \zeta_2, \cdots, \zeta_C]^{\top}\in\mathbb{R}^C$ are the \emph{clustering size controlling variables} against unbalanced data distributions~\cite{miyamoto2008algorithms}.~\cref{eq:loss_cluster} forms the basis of our method for the fitting of  both rigid and deformable superquadrics. We experimentally demonstrate its superiority over prior approaches.

\section{Proposed Method}\label{sec:method}
In this section, we present specific steps to show how to formulate the superquadric fitting problem as an unsupervised clustering process. 

\subsection{Overview}
We first formally define the superquadric fitting problem in~\cref{sec:problem_formulation}, establishing the mathematical framework for subsequent analysis. Then, in~\cref{sec:parameter_fitting}, we detail the superquadric parameter estimation process, elaborating on the motivation for adopting unsupervised clustering and addressing the challenges of inefficient, non-uniform surface sampling. In~\cref{sec:advantages}, we analyze the clustering framework's merits from a convex optimization perspective, highlighting its inherent mechanism for escaping local minima. Finally, we present the unsupervised optimization pipeline in~\cref{sec:unsupervised_optimization}, deriving analytical closed-form solutions for the fuzzy membership matrix and the covariance matrix to enable efficient and scalable parameter estimation. Finally, in~\cref{sec:cluster_convergence}, we provide a theoretical analysis of the algorithm, establishing its convergence guarantees.

\subsection{Problem Formulation}\label{sec:problem_formulation} Given a set of input point clouds $\mathbf{X}=\{\bm{x}_i\in \mathbb{R}^3\}_{i=1}^m$ sampled from a superquadric surface, potentially contaminated by noise and outliers, as defined in~\cref{subsec:SQ}, our goal is to estimate the optimal parameter set $\widehat{\mathbf{\Theta}}$.
This set defines the underlying superquadric surface \(S_{\widehat{\mathbf{\Theta}}}\) that minimizes the deviation between \(\mathbf{X}\) and \(S_{\widehat{\mathbf{\Theta}}}\).  For simplicity, \(\widehat{\mathbf{\Theta}}\) denotes the parameter vector encompassing both rigid and deformable superquadric cases.

\begin{figure}[t]
\centering
\includegraphics[width=0.87\linewidth]{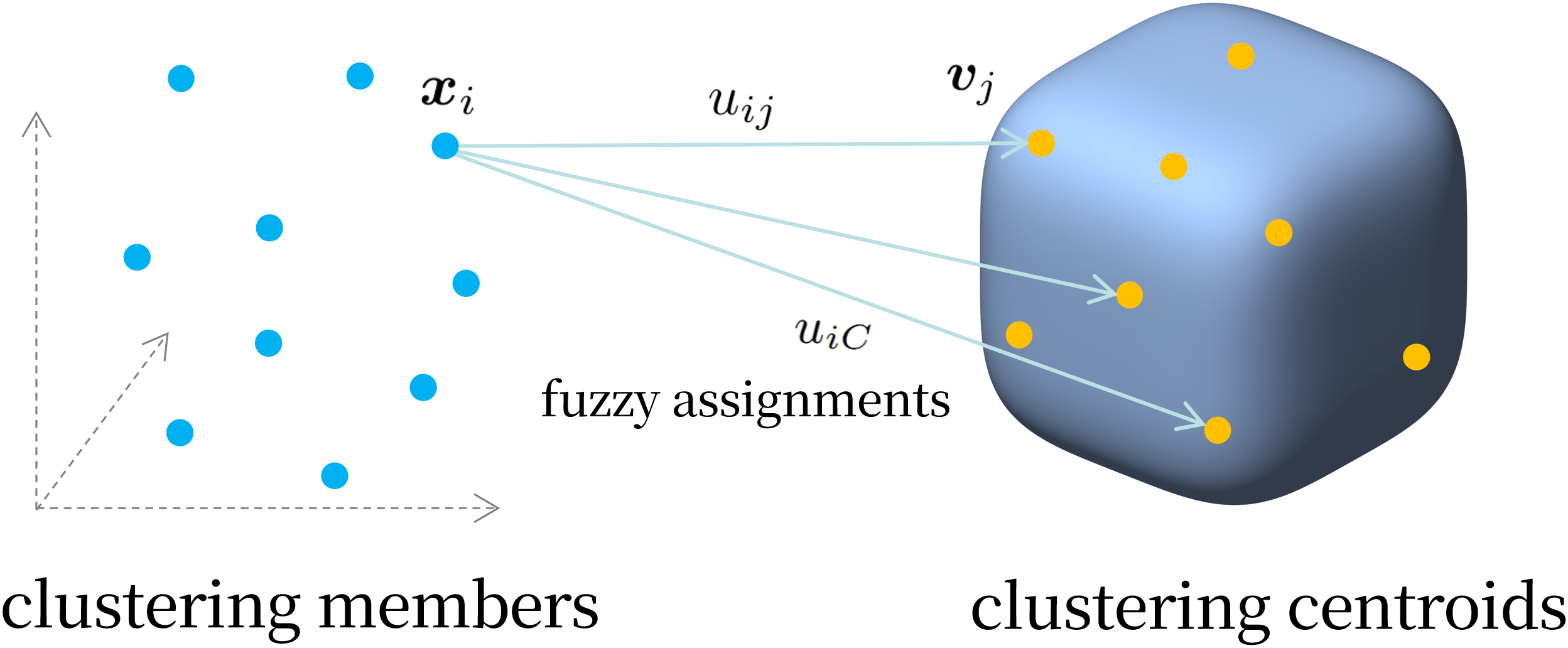}
\caption{Visualization of the fuzzy clustering-inspired superquadric fitting framework. The input point cloud \(\mathbf{X} = \{\bm{x}_i \in \mathbb{R}^3\}_{i=1}^M\) serves as clustering members, while \(\mathbf{V} = \{\bm{v}_j \in \mathbb{R}^3\}_{j=1}^C\) represents the clustering centroids approximating the underlying superquadric surface. Analogous to clustering process, the iterative update of centroids \(\mathbf{V}\) is mathematically equivalent to optimizing the superquadric parameters \(\mathbf{\Theta}\), where each centroid \(\bm{v}_j\) lies on the surface, \ie,  \(S_\mathbf{\Theta}(\bm{v}_j) = 0\).}
\label{fig:clustering_vis}
\end{figure}
\subsection{Parameter Estimate via a Clustering    Framework}\label{sec:parameter_fitting}
\subsubsection{Motivation} During the clustering process, we observe that the spatial positions of clustering centroids \(\mathbf{V} \subset \mathbb{R}^3\) are iteratively updated until the distance between \(\mathbf{V}\) and their corresponding members \(\mathbf{X} \subset \mathbb{R}^3\) is minimized (\ie, \cref{eq:loss_cluster} reaches its minimum). This dynamic update mechanism shares striking similarities with the iterative optimization of superquadrics, where each sample point \(\bm{x}_i \in \mathbb{R}^3\) is associated with its corresponding centroid \(\bm{v}_j \in \mathbb{R}^3\) lying on the potential superquadric surface, \ie, \(S_\mathbf{\Theta}(\bm{v}_j) = 0\). Inspired by this analogy, we propose to reformulate the superquadric fitting problem as an \emph{unsupervised clustering optimization process}:

\begin{equation}
\begin{gathered}
\min_{\mathbf{U},\mathbf{\Theta},\mathbf{\Sigma}}\mathcal{L}(\mathbf{U},\mathbf{\Theta},\mathbf{\Sigma})\begin{aligned}=\sum_{j=1}^C\sum_{i=1}^M{u}_{ij}||\mathbf{\Sigma}_j^{-\frac{1}{2}}(\bm{x}_i-\bm{v}_j(\mathbf{\Theta}))||_2^2\end{aligned} \\
\begin{aligned}+{u}_{ij}\text{log}|\mathbf{\Sigma}_j|+\lambda {u}_{ij}(\text{log}{{u}_{ij}}+\text{log}C),\end{aligned} \\
s.t.~|\mathbf{\Sigma}_j|=\theta_j, \sum_{j=1}^C{u}_{ij}=1, u_{ij}\in[0,1]. 
\end{gathered}
\label{eq:loss_reg}
\end{equation} We consider \(\mathbf{V} = \{\bm{v}_j \in \mathbb{R}^3\}_{j=1}^C\) as the clustering centroids that approximate the potential  continuous superquadric surface, and the input point clouds \(\mathbf{X} = \{\bm{x}_i \in \mathbb{R}^3\}_{i=1}^M\) as the clustering members. Here, we set \(p = 2\) and \(\zeta_j \equiv 1/C\) to simplify the computation, which also ensures closed-form solutions for the fuzzy membership degree matrix \(\mathbf{U} \in \mathbb{R}^{M \times C}\) and the covariance matrix \(\mathbf{\Sigma}_j \in \mathbb{S}^{n}_{++}\) as derived in the following. For clarity,~\cref{fig:clustering_vis} visualizes the clustering-inspired fitting process.

\subsubsection{How to Avoid Surface Sampling?} 
We note that computing the fuzzy membership degree matrix $\mathbf{U}$ in~\cref{eq:loss_reg} requires pairwise distance computations between each clustering sample $\bm{x}_i\in\mathbb{R}^3$ and all centroids $\bm{v}_j\in\mathbb{R}^3$ for $ j=1, \cdots, C$. Nevertheless, there is no rigorous or theoretically sound method for equal-distance sampling of $\bm{v}_j$, whether via parametric or explicit approaches, from the underlying superquadric $S_\mathbf{\Theta}$~\cite{liu2022robust}. Directly sampling $C$ points from $S_\mathbf{\Theta}$ can lead to \emph{non-uniform distribution} and \emph{inefficient sampling} problems~\cite{vaskevicius2017revisiting}. To address these challenges, we derive and analyze the relationship between pairwise distance computations and the geometrically closest (\ie, orthogonal) distance.

\begin{lemma}[\cite{charalambous1976non}]
Given a set of real numbers  $d_{i}\geq0$  for  $i=1, \cdots, N-1$, and a variable  $d_{N}\in\mathbb{R}_{+}$, then,
\begin{equation}
J=\left(\sum_{i=1}^{N-1} d_{i}^{-v}+d_{N}^{-v}\right)^{-1 / v}, v \geq 1    
\end{equation}
decreases as  $d_{N}$  decreases.
\label{lemma:1}
\end{lemma}

\begin{lemma}[\cite{charalambous1976non}]
If $d_i \geq 0$ for $i\in[N]$, and $v>0$, then,
\begin{small}
\begin{equation}
\min _{i \in[N]} \frac{d_i}{N^{1 / v}} \leq\left(\sum_{i=1}^N d_i^{-v}\right)^{-1 / v} \leq \min _{i \in[N]} d_i.
\end{equation} 
\end{small}
\label{lemma:2}
\end{lemma}
Building upon~\cref{lemma:1} and~\cref{lemma:2}, we establish the equivalence between pairwise distance and closest distance computations through the following proposition.

\begin{proposition}
The sum of weighted pairwise distances from each clustering sample $\bm{x}_i\in \mathbb{R}^n$ to all clustering centroids $\mathbf{V}=\{\bm{v}_j\in\mathbb{R}^n\}$ in \cref{eq:fcm_origin} is equivalent to the closest distance from $\bm{x}_i$ to $\mathbf{V}$ when the fuzzy factor $r\rightarrow1$, i.e., 
\begin{equation}
\lim_{r\rightarrow1}\left(\sum_{j=i}^C(u_{ij})^r||\bm{x}_i-\bm{v}_j||_2^2\right)= \min_{j\in[C]}||\bm{x}_i-\bm{v}_j||_2^2. 
\end{equation}

\label{proposition:1}
\end{proposition}We present the theoretical proof of~\cref{proposition:1} in the \emph{Supplementary Material}. From~\cref{proposition:1}, we mitigate the need for iterative point sampling by \revise{reformulating it as the closest distance computation}. We utilize the closed-form radial distance that refines the mean-square inside-outside measure~\cite{gross1988error} to approximate the geometric distance, as an analytical solution for the exact point to superquadric distance remains intractable. For each sample point \(\boldsymbol{x}_i\), its corresponding closest centroid \(\boldsymbol{v}_i(\mathbf{\Theta})\) on the superquadric \(\mathcal{S}(\mathbf{\Theta})\) can be analytically derived as:
\begin{equation}
\bm{v}_i(\mathbf{\Theta})=\bm{x}_i-\left(1\!-\!F^{-\frac{\epsilon_1}{2}}\left(\mathcal{D}^{-1} (\bm{x}_i)\right)\right)\mathcal{D}^{-1} (\bm{x}_i)\in\mathbb{R}^3,   
\label{eq:closest_point}
\end{equation} where $\mathcal{D}^{-1}$
represents the \emph{inverse mapping} of the spatial and deformable transformations (\ie, rigid, tapering, or bending) applied to the superquadrics (details are presented in the \emph{Supplementary Material}).~\cref{eq:closest_point} enables direct computation of the point \(\bm{v}_i(\mathbf{\Theta})\) on the superquadric surface without recourse to discrete sampling.  Then, \cref{eq:loss_reg} can be simplified to 
\begin{equation}
\begin{gathered}
\min_{\mathbf{u},\mathbf{\Theta},\mathbf{\Sigma}}\mathcal{L}({\mathbf{u},\mathbf{\Theta},\mathbf{\Sigma}}) \!=\!\sum_{i=1}^M{u}_{i}||\mathbf{\Sigma}_i^{-\frac{1}{2}}(\bm{x}_i\!-\!\bm{v}_i(\mathbf{\Theta}))||_2^2\!+\!{u}_{i}\text{log}|\mathbf{\Sigma}_i|\\
+\lambda {u}_{i}(\text{log}{{u}_{i}}+\text{log}M).\\
s.t.~|\mathbf{\Sigma}_j|=\theta_j, ~ u_{i}\in[0,1], 
\end{gathered}
\label{eq:loss_fit_final}
\end{equation} where $\mathbf{u}=[u_1, u_2, \cdots, u_M]^{\top}\in\mathbb{R}^{M}$ is the degenerated fuzzy membership degree vector.

\subsection{Advantages of the Clustering-Inspired Method}\label{sec:advantages}
We further analyze the advantages of the clustering-induced fitting method in \cref{eq:loss_fit_final}. From an optimization perspective, the term \( f(u_i) = u_i \log(u_i) \) is strictly convex with respect to \( u_i \), as its second derivative
$\nabla^2 f(u_i) = \frac{1}{u_i} > 0,$ where \( \nabla \) denotes the gradient operator. Consequently, \( f(u_i) \) introduces \emph{convexity} to the objective function defined in \cref{eq:loss_fit_final}, which helps to guide the optimization away from local minima. The parameter 
\(\lambda\) controls the influence of this convex term, thereby regulating the landscape of the solution space. We conduct an ablation study to investigate the impact of \(\lambda\) in the subsequent experiments and show that increasing $\lambda$ \revise{effectively facilitates escaping} from local minima. Moreover, $f(u_i)$ functions as a \emph{barrier function}, enforcing the constraint that $u_i$ stays within the interval $[0, 1]$. This constraint is inherently justified as $f(u_i)$ attains its global minimum within the interval $[0, 1]$.

\subsection{Unsupervised Optimization}\label{sec:unsupervised_optimization}
We adopt the \emph{coordinate descent} (\emph{alternating optimization}) framework to solve for $\mathbf{u}$, $\mathbf{\Theta}$, and $\mathbf{\Sigma}$ in a decoupled manner. Moreover, our method enables analytical solutions for $\mathbf{u}$ and $\mathbf{\Sigma}$ as derived below, thus guaranteeing efficient implementation.

\begin{figure*}
    \centering
    \includegraphics[width=\linewidth]{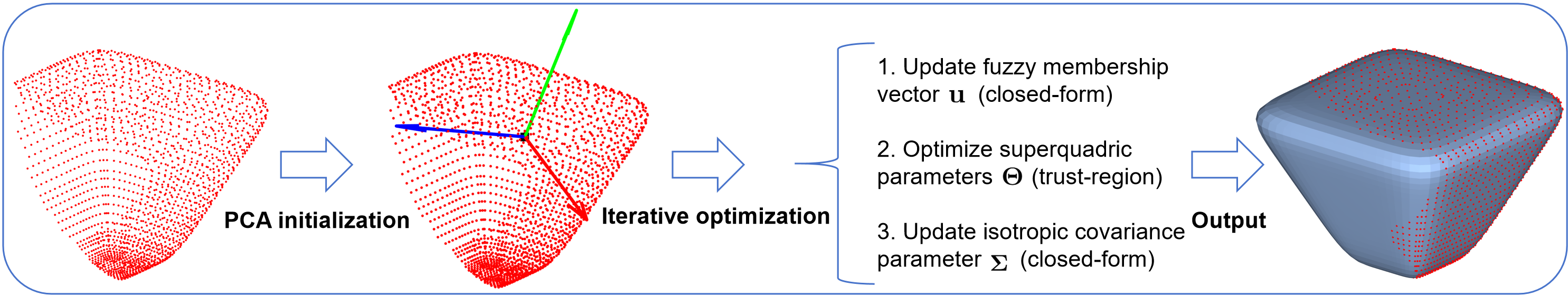}
    \caption{\revise{Flowchart of the proposed superquadric fitting method. Given an input point cloud, we first perform PCA initialization, followed by iterative optimization, where the fuzzy membership vector $\mathbf{u}$, superquadric parameters $\mathbf{\Theta}$, and covariance matrix $\mathbf{\Sigma}$ are updated sequentially. The algorithm outputs the fitted superquadric upon convergence.}}
    \label{fig:flowchart}
\end{figure*}

\subsubsection{Update of $\mathbf{u}$} First, we fix $\mathbf{\Theta}$ and $\mathbf{\Sigma}$ to update $\mathbf{u}$, which reduces to a convex optimization problem. By setting $\frac{\partial\mathcal{L}}{\partial\mathbf{u}}=0$, we obtain the following closed-form solution: 
\begin{equation}
\mathbf{u}=\frac{\exp{(-\mathbf{d}/\lambda)}\odot |\mathbf{\Sigma}|^{-1/\lambda}}{Me}\in\mathbb{R}^M,
\label{eq:U}
\end{equation}where $\mathbf{d}=[d_i]\in \mathbb{R}^{M}$ is a squared Mahalanobis distance vector defined by \(d_{i} =\|\mathbf{\Sigma}_j^{-\frac{1}{2}} (\bm{x}_i - \bm{v}_j(\mathbf{\Theta}))\|_2^2\). The function \(\mathrm{exp}(\cdot)\) denotes the element-wise exponential operator on vectors, \(|\mathbf{\Sigma}|^{-1/\lambda} = [|\mathbf{\Sigma}_1|^{-1/\lambda}, \cdots, |\mathbf{\Sigma}_M|^{-1/\lambda}]^\top \in \mathbb{R}^C\) represents the vector of covariance determinants \(|\mathbf{\Sigma}_j|\) for all clustering centroids, and $\odot$ signifies the element-wise Hadamard product of two vectors.  To ensure \(\mathbf{u} \in [0, 1]^M\) and further enhance the algorithm's robustness against noise and outliers, like~\cite{wells1997statistical}, we incorporate an additional uniform distribution to account for outliers. Consequently, \(\mathbf{u}\) becomes
\begin{equation}
\text{diag}(\mathbf{u} + \mathbf{p})^{-1} \mathbf{u}\in[0, 1]^M,
\label{eq:U_normalization}
\end{equation} where \(\mathrm{diag}(\cdot)\) denotes the diagonal matrix with its argument on the main diagonal, \(\mathbf{p} = [p_i] \in \mathbb{R}^M\), with \(p_i = {w \sqrt{(2\pi)^3 |\Sigma_i|}}/{(1-w)V}\in\mathbb{R}_{\geq0}\). Here, \(V\) represents the volume of the bounding box of the input point samples \(\mathbf{X}\), and \(w \in \mathbb{R}_{\geq 0}\) is the weight factor that balances the influence of points from \(\mathcal{S}_{\mathbf{\Theta}}\) and outliers.

\subsubsection{Update of $\mathbf{\Theta}$} 
After updating $\mathbf{u}$, we proceed to optimize the set of unknown variables $\mathbf{\Theta}$. By omitting the terms that are independent of $\mathbf{\Theta}$ in~\cref{eq:loss_fit_final}, and given that~\cref{eq:closest_point} represents a non-linear transformation, \cref{eq:loss_fit_final} becomes a \emph{weighted non-linear least-squares problem}. Similar to~\cite{liu2022robust}, we adopt the \emph{interior trust region}
method~\cite{coleman1996interior}, which 
ensures both global convergence and local quadratic convergence for optimizing $\mathbf{\Theta}$, particularly for the high-dimensional solution space including shape, location, and deformation parameters. {Compared to the Levenberg–Marquardt scheme~\cite{more2006levenberg}, which guarantees only local convergence with a linear or superlinear rate. The interior trust region approach achieves faster convergence. Moreover, it is more stable because it adaptively controls the search region and maintains a better step size.}
The upper and lower bounds for each parameter setting used in the interior trust region optimization are detailed in the \emph{Supplementary Material.}

\subsubsection{Update of $\mathbf{\Sigma}$} 
Finally, we optimize the fuzzy covariance matrix \(\mathbf{\Sigma}\). For simplicity, we relax each clustering centroid's covariance matrix to be isotropic, \textit{i.e.}, \(\mathbf{\Sigma}_j = \sigma^2 \mathbf{I}\), where \(\mathbf{I} \in \mathbb{R}^{3 \times 3}\) is the identity matrix. This simplication ensures a closed-form solution for the variance \(\sigma^2\):
\begin{equation}
\sigma^2 = \frac{\operatorname{Tr}\left(\mathrm{diag}(\mathbf{u})(\mathbf{X} - \mathbf{V})(\mathbf{X} - \mathbf{V})^\top\right)}{3 \mathbf{u}^\top \mathbf{1}_M},
\label{eq:sigma}
\end{equation}where $\mathbf{1}_M\in\mathbb{R}^M$ represents the all-ones vector. Although the use of non-isotropic covariance matrices leads to a more complex optimization problem, studies have shown that there is no  significant improvement in accuracy~\cite{evangelidis2017joint}. The above parameters are iteratively updated until the algorithm converges. 
\begin{algorithm}[t]
\SetAlgoLined
\KwIn{Point cloud samples $\mathbf{X}=\{\bm{x}_i\in \mathbb{R}^3\}_{i=1}^M$, regularization factor $\lambda\in \mathbb{R}_{\geq0}$ and weight factor $w\in \mathbb{R}_{\geq0}$. }
\KwOut{The estimated superquadric parameter $\mathbf{\widehat{\Theta}}.$}
Perform PCA on $\mathbf{X}$ to initialize the size and orientation parameters.\\
Initialize $\mathbf{t}_0=\texttt{mean}(\mathbf{X})$.\\
Compute the volume $V$ of the bounding box of the input point clouds \(\mathbf{X}\) and initialize $\sigma^2=V^{1/3}$.\\
\While{not converging}{
Update $\mathbf{u}$ by~\cref{eq:U_normalization}.\\
Update $\mathbf{\Theta}$ by interior trust region optimization.\\
Update $\sigma^2$ by~\cref{eq:sigma}.\\
}
\caption{Pseudo-code for Superquadric Fitting }
\label{alg:iteration}
\end{algorithm}

\subsubsection{Optimization Initialization}
The proposed method requires initialization of the unknown parameter $\mathbf{\Theta}$ due to its alternating optimization strategy.  Following~\cite{vaskevicius2017revisiting,liu2022robust}, {we adaptively initialize the size parameters $a_x$, $a_y$, $a_z$ as the median lengths of the point cloud along the $x$, $y$, $z$ axes.} \revise{The rotation matrix}, represented by a set of 3D Euler angles, by performing \emph{Principal Component Analysis} (PCA) on the input point cloud $\mathbf{X}$. The position $\mathbf{t}$ is initialized as the centroid of the point cloud, \ie, $\mathbf{t}_0=\texttt{mean}(\mathbf{X})$. For rigid and tapering superquadrics, the shape parameters are set to $\epsilon_1 = \epsilon_2 = 1$. While for bending superquadrics, we initialize $\epsilon_1=0.01$ and $\epsilon_2=1$. The determinant $|\mathbf{\Sigma}_j| = |\sigma^2 \mathbf{I}|=(\sigma^2)^3$ is geometrically related to the volume $V$ of the bounding box of the input point samples $\mathbf{X}$, therefore we initialize the scale parameter as $\sigma^2 = V^{1/3}$. {We summarizes the optimization process in~\cref{alg:iteration} and illustrate it with the flowchart in~\cref{fig:flowchart} for ease of understanding.}

\subsection{Theoretical Convergence Analysis}\label{sec:cluster_convergence}
Subsequently, we present the theoretic analysis of the algorithm convergence.

\begin{proposition}
The proposed superquadric fitting method summarized in~\cref{alg:iteration} monotonically decreases the objective function defined in \cref{eq:loss_fit_final} in each iteration step until convergence. 
\end{proposition}
\begin{proof}
\cref{alg:iteration} takes the alternating optimization strategy to update $\mathbf{u}$, $\mathbf{\Theta}$, and ${\sigma}^2$, separately. Therefore, for the objective function in~\cref{eq:loss_fit_final},~\cref{alg:iteration} will converge if 
each variable update leads to a non-increasing function value. Let $\hat{\mathbf{u}}_{t}$, $\widehat{\mathbf{\Theta}}_{t}$, and $\hat{{\sigma}}^2_{t}$ denote the variables at the $t$-th iteration. Then, in the $(t+1)$-th iteration, when we fix other variables to update $\mathbf{u}$,~\cref{eq:loss_fit_final} degenerates to a convex problem, and thus we can obtain a global optimal $\mathbf{u}_{t+1}$ for~\cref{eq:loss_fit_final}, which also means $\mathbf{u}_{t+1}$ can better decrease the objective function in~\cref{eq:loss_fit_final} than $\mathbf{u}_{t}$. When we fix $\mathbf{u}$ and ${{\sigma}^2}$ to update $\mathbf{\Theta}$, as the trust region method guarantees global convergence~\cite{powell1975convergence,shultz1985family}, $\mathbf{\Theta}_{t+1}$ yielding a lower objective function value in~\cref{eq:loss_fit_final}  than $\mathbf{\Theta}_{t}$. Similarly, when we update ${\sigma}^2$ with fixed 
$\mathbf{u}$ and $\mathbf{\Theta}$,~\cref{eq:loss_fit_final} also degenerates to a convex problem, yielding a global optimal ${\sigma}^2_{t+1}$ that further decreases the objective function~\cref{eq:loss_fit_final} relative to \(\sigma^2_{t}\). Since each update step decreases the objective function,~\cref{alg:iteration} converges monotonically. This completes the proof.
\end{proof}

We perform experiments to show the convergence behavior of the designed superquadric fitting algorithm in the following.

\section{Experimental Evaluations}\label{sec:experiments}
In this section, we conduct extensive experiments to validate the performance of the designed superquadric fitting method for both rigid and deformable objects. We first detail the implementation setup in \cref{sec:experiment_initialization} and define the evaluation metric in \cref{subsec:evaluation_metric}. Subsequently, we assess the fitting accuracy \revise{and robustness} for rigid superquadrics in \cref{sec:experiment_rigid} and  for deformable superquadrics in \cref{sec:experiment_deformable}. The method is further validated on real-world scanned data in \cref{sec:experiment_real}. We then analyze the algorithm's convergence behavior in \cref{sec:experiment_convergence} and conduct an ablation study on the influence of the $u_i\log u_i$ regularization term for escaping local minima in \cref{sec:experiment_ablation}. Finally, we demonstrate the method's applicability to type-specific primitive fitting in \cref{sec:type_specific_fitting} and its versatility in downstream applications including complex shape \revise{approximation}, geometry editing, and medical modeling in \cref{sec:experiment_applications}.

\subsection{Implementation Details}\label{sec:experiment_initialization}
Given a set of samples \(\mathbf{X} = \{\bm{x}_i \in \mathbb{R}^3\}_{i=1}^{M}\), we first normalize the point coordinates to increase the numerical stability for optimization. However, the evaluation is still based on the original input by denormalization. During optimization, \revise{we fix the regularization factor \(\lambda = 3\) in~\cref{eq:loss_fit_final}}, which yields impressive and stable results across various settings. All algorithms and experiments are implemented in MATLAB and executed on a laptop with an 11th Gen Intel(R) Core(TM) i7-11800H(2.30GHz) processor. We leverage publicly available implementations of baseline approaches for assessment, with their parameters either fine-tuned by us or configured to the values recommended by the authors to achieve optimal performance. Concrete parameter settings and implementation details of baseline
approaches are reported in the \emph{Supplementary Material}.

\subsection{Evaluation Metric}\label{subsec:evaluation_metric}
As in~\cite{liu2022robust}, we employ the following metric to quantitatively evaluate the efficacy of the fitted superquadric \(\mathcal{S}_{\widehat{\mathbf{\Theta}}}\):
\begin{equation}
dist(\mathbf{X}, \mathcal{S}_{\widehat{\mathbf{\Theta}}}) = \mathbb{E}_{\bm{s}_i\sim\mathcal{S}_{\widehat{\mathbf{\Theta}}}}\|\bm{x}_i - \bm{s}_i\|_2,
\label{eq:metric}
\end{equation}where \(\bm{s}_i = \min_{k\in[K]} \|\bm{x}_i - \bm{s}_k\|_2\) and \(\{\bm{s}_k\}_{k=1}^K\in \mathcal{S}_{\widehat{\mathbf{\Theta}}}\) represent a set of densely and evenly sampled points (using nearly equal-distance sampling) on the fitted superquadric \(\mathcal{S}_{\widehat{\mathbf{\Theta}}}\).

\begin{figure*}[t]
    \centering
\includegraphics[width=0.96\linewidth]{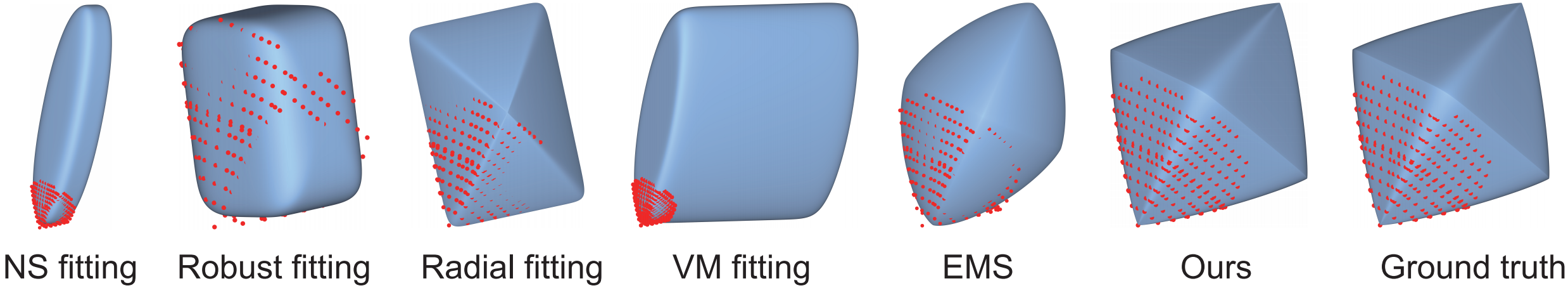}
\caption{Qualitative comparisons of rigid superquadric fitting under severe occlusion with a partial ratio of $r=0.3$. Red point clouds represent the input partial data. Our method achieves most accurate fitting than competing approaches.}
\label{fig:rigid_test_vis}
\end{figure*}

\begin{figure}[t]
\centering
\includegraphics[width=0.495\linewidth]{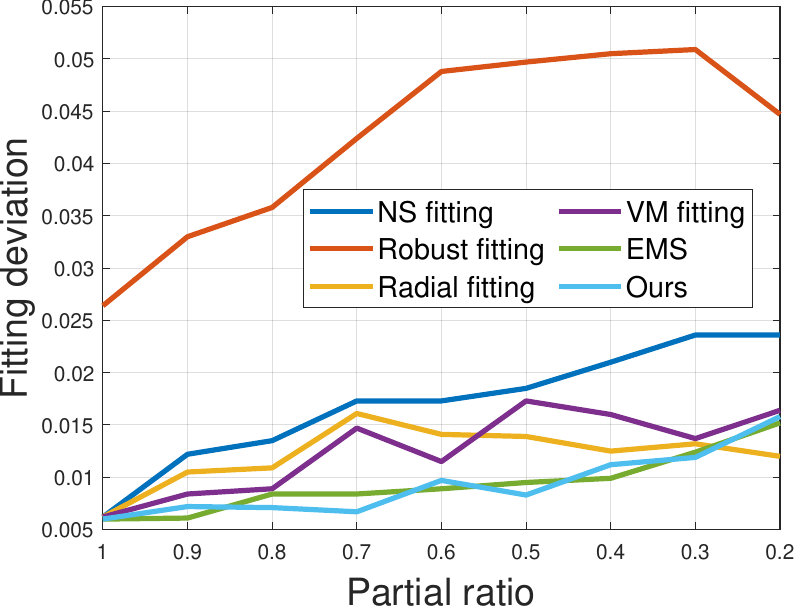}
\includegraphics[width=0.48\linewidth]{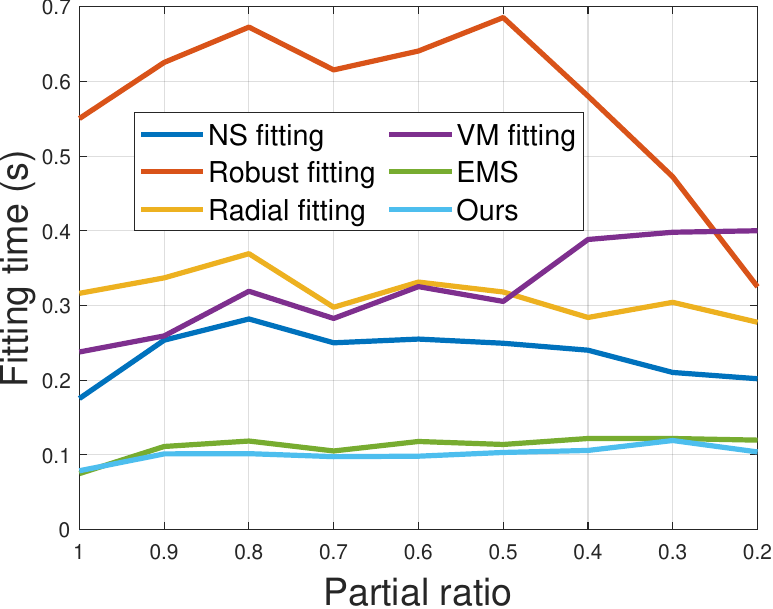}
\caption{Quantitative comparisons of accuracy (left) and efficiency (right) for rigid superquadric fitting  under occlusion disturbances. We report the average result across 100 random tests for each partial ratio.}
\label{fig:rigid_test}
\end{figure}

\begin{figure}[t]
\centering
\includegraphics[width=0.495\linewidth]{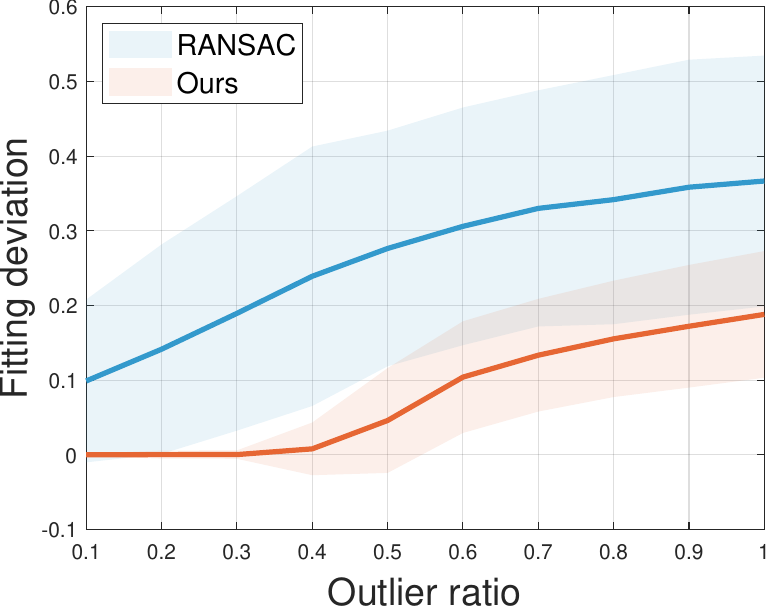}
\includegraphics[width=0.495\linewidth]{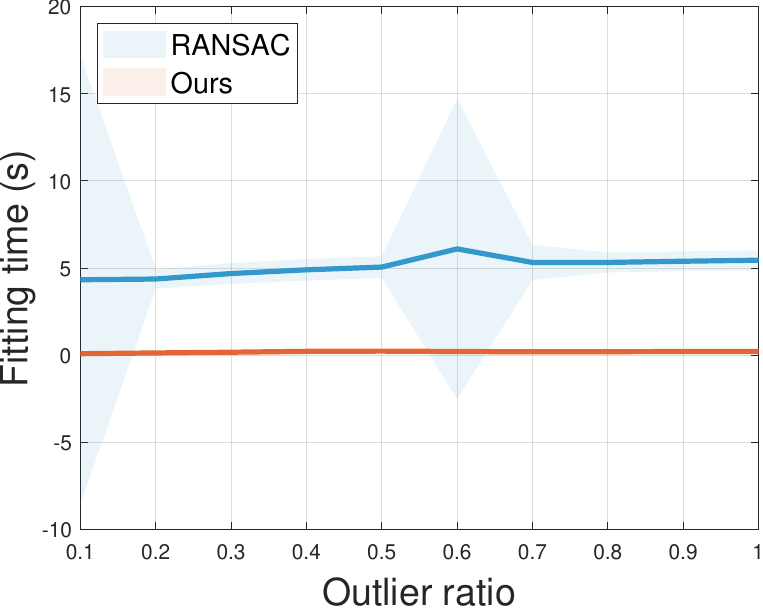}
\caption{\revise{Robustness comparison in terms of accuracy (left) and efficiency (right) between RANSAC and the proposed method under increasing outlier ratios. For each outlier ratio, we report the average result over 1,000 random trials, along with the standard deviation (shown as the bandwidth).}}
\label{fig:rigid_test_outlier}
\end{figure}
\begin{figure*}[t]
    \centering
\includegraphics[width=\linewidth]{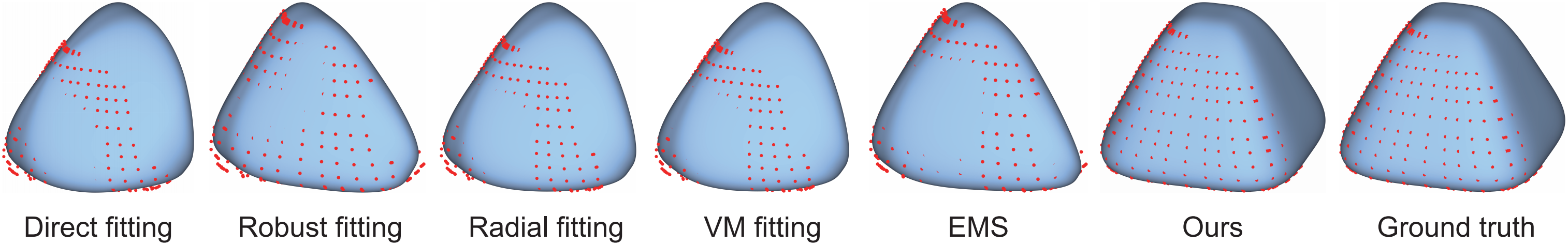}
\caption{Qualitative comparisons of deformable superquadric fitting under tapering deformations with occlusion.}
\label{fig:tapering_test_vis}
\end{figure*}
\begin{figure*}[t]
    \centering
\includegraphics[width=\linewidth]{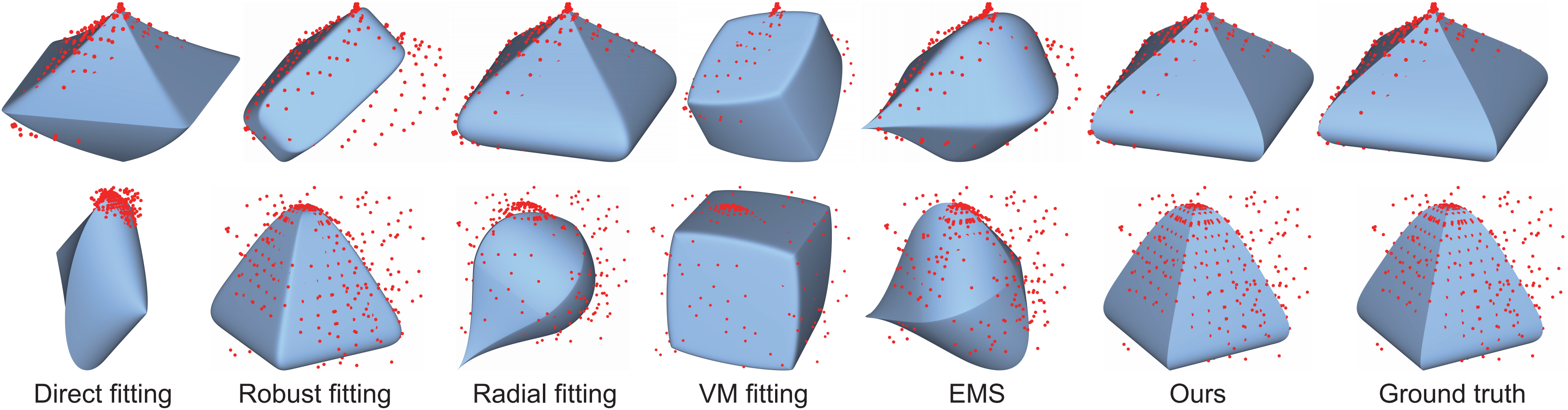}
\caption{Robustness comparisons of deformable tapering superquadric fitting under noise (top) and outlier (bottom) disturbances. As observed, the proposed method delivers highly accurate and stable fittings.}
\label{fig:tapering_test_robust_vis}
\end{figure*}
\subsection{Assessment of Rigid Superquadric Fitting}\label{sec:experiment_rigid}
\subsubsection{Parameter Setting} \label{subsubsec:rigid_parameter_setting}
To assess the proposed method on rigid superquadric fitting, particularly under \emph{partially visible data}, a common challenge in real-world scenarios, we consider various partial ratios \( r \in [0.2, 1] \) with a step size of \( \Delta r = 0.1 \). This means only $r\times 100\%$ of the superquadric's surface is accessible for fitting. For each partial ratio, we randomly generate 100 superquadrics, resulting in a total of 900 occluded superquadrics for evaluation. Analogous to~\cite{liu2022robust}, the parameters of the generated superquadrics are determined using the following strategy: The shape parameters \(\epsilon_1\) and \(\epsilon_2\) are uniformly sampled from the interval \((0, 2]\), \textit{i.e.}, $(\epsilon_1, \epsilon_2) \in (0, 2]^2$, while the size parameters \(a_x, a_y, a_z\) are sampled from the interval \([0.5, 3]\). The rotation matrix \(\mathbf{R}\) is determined by randomly sampling a rotation axis on the unit sphere, and the translation vector \(\mathbf{t}\) is sampled from the interval \([-1, 1]^3\). All test samples are generated by evenly sampling points on the random superquadric surface with a constant interval of $0.2$.

\subsubsection{Quantitative Results} The quantitative comparison results with state-of-the-art and representative approaches, including NS fitting~\cite{vaskevicius2017revisiting}, Robust fitting~\cite{hu1995robust}, Radial fitting~\cite{gross1988error}, VM fitting~\cite{solina1990recovery}, and EMS~\cite{liu2022robust}, are reported in~\cref{fig:rigid_test}. For each partial ratio $r$, we report the average fitting error over 100 random superquadrics and the average computational time for each method. As observed, all approaches exhibit increasing fitting error with higher occlusion levels. However, our method consistently achieves the best or comparable fitting accuracy relative to competing approaches, while also demonstrating superior efficiency across all levels of occlusion. \cref{fig:rigid_test_vis} presents a qualitative comparison result under severe occlusion (\eg, \( r = 0.3 \)), where our method still delivers the most accurate fitting and the highest fidelity to the ground truth surface.

\begin{figure}[!ht]
\centering
\includegraphics[width=0.489\linewidth]{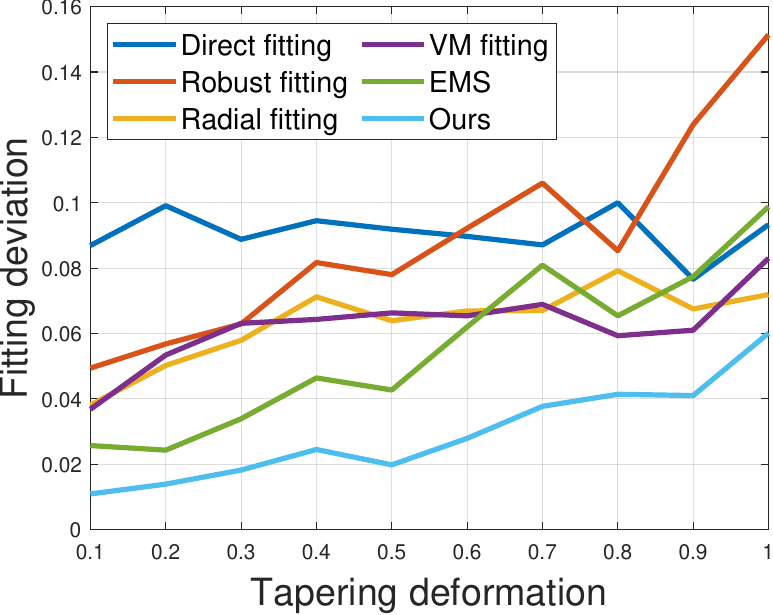}
\includegraphics[width=0.494\linewidth]{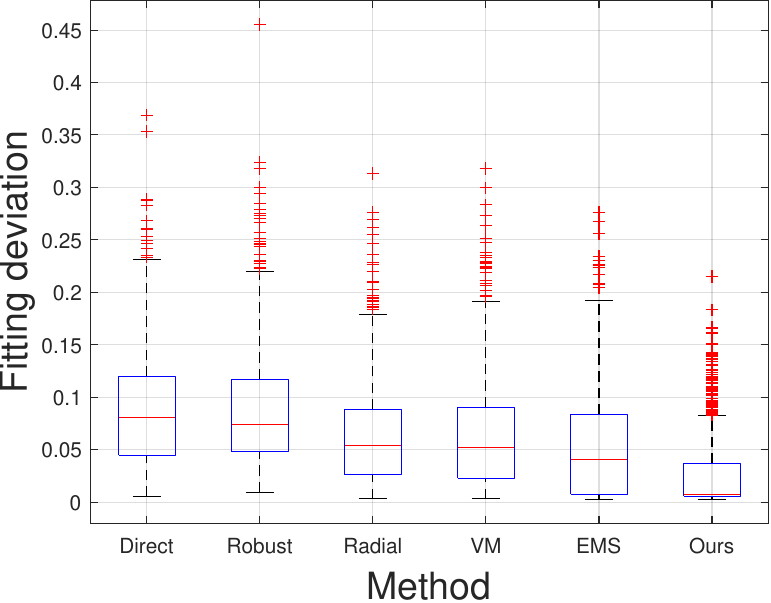}
\caption{Quantitative comparisons for fitting tapering deformable superquadrics. Left: Results for 1,000 randomly generated superquadrics with increasing tapering deformations ($k_1=k_2$); Right: Results for an additional 1,000 random superquadrics with $k_1$ and $k_2$ uniformly sampled from $[0, 1]$ (allowing non-equal values).}
\label{fig:tapering_test}
\end{figure}
\revise{\subsubsection{Robustness against Outliers}}
\revise{Next, we evaluate the robustness of the proposed method against outliers. We re-implement the classic RANSAC-based superquadric fitting approach~\cite{afanasyev20123d} for comparison,} {where the minimum number of points required to fit a superquadric is set to 13 (we found that using 11 points yields poor results). The outlier ratio varies from $0.1$ to $1.0$ in steps of $0.1$ relative to the original point cloud size. For each outlier ratio, we randomly generate 1,000 rigid superquadrics using the same parameter settings as in Section~\ref{subsubsec:rigid_parameter_setting}. Fig.~\ref{fig:rigid_test_outlier} reports the average accuracy and computation time, along with their standard deviations, for each outlier ratio. As observed, compared to RANSAC, our method exhibits stronger robustness, consistently achieving lower fitting deviation. Moreover, our method demonstrates greater stability, as reflected by the narrower bandwidth. Unlike RANSAC, which requires hundreds of sampling iterations and thus consumes more time, our method delivers substantially more efficient fitting.}

\subsection{Assessment of Deformable Superquadric Fitting}\label{sec:experiment_deformable}
Subsequently, we evaluate the performance of the proposed method on the fitting of more challenging deformable superquadric surfaces. Given the generality of our optimization framework, we test it on both tapering and bending deformations.

\begin{figure*}[!htbp]
    \centering
\includegraphics[width=0.96\linewidth]{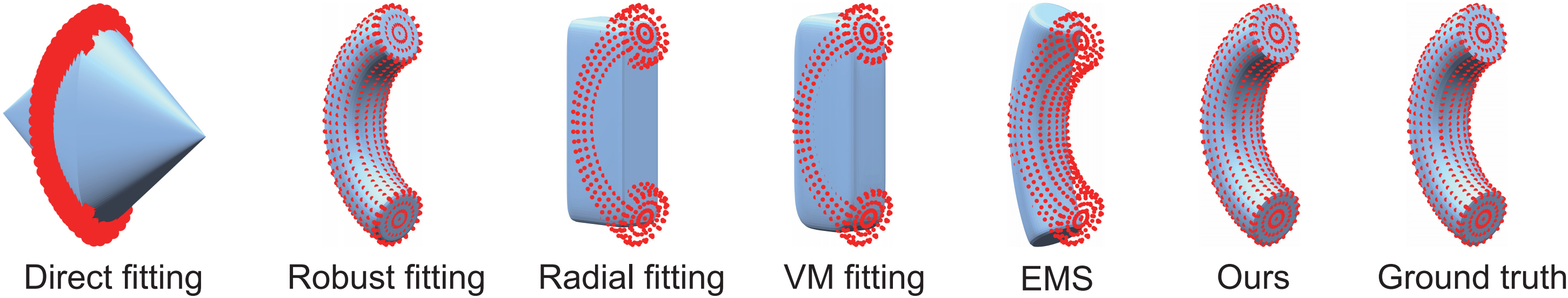}
\caption{Qualitative comparisons of deformable superquadric fitting under bending deformations.}
\label{fig:bending_test_vis}
\end{figure*}

\begin{figure}[t]
\centering
\includegraphics[width=0.494\linewidth]{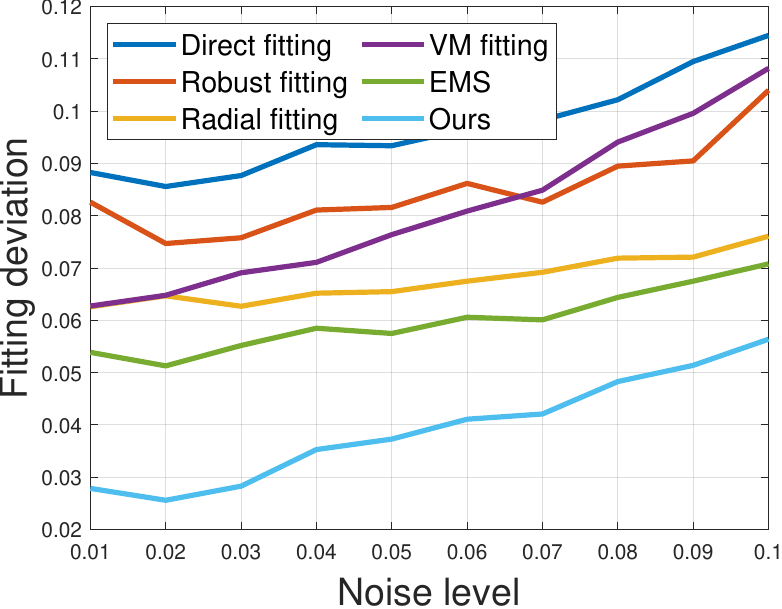}
\includegraphics[width=0.495\linewidth]{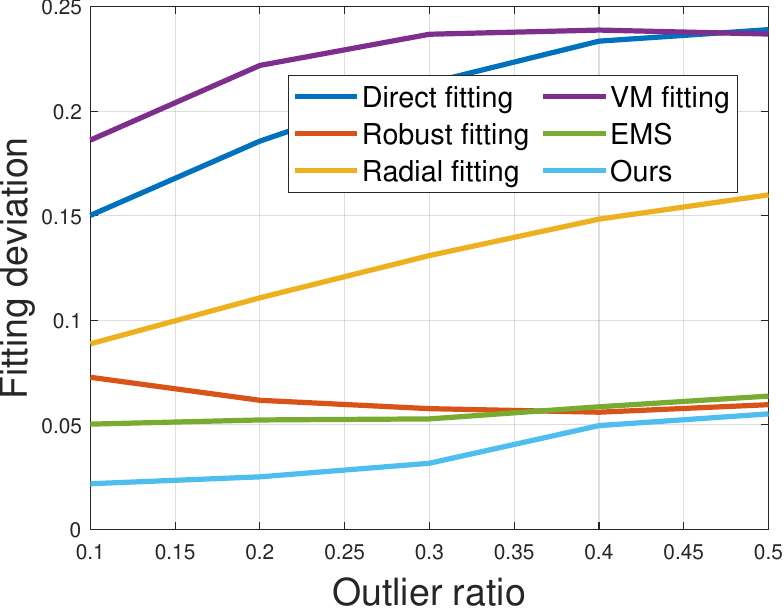}
\caption{Robustness comparisons against external disturbances, including noise (left) and outliers (right).} 
\label{fig:tapering_test_robust}
\end{figure}

\begin{figure}[t]
\centering
\includegraphics[width=0.493\linewidth]{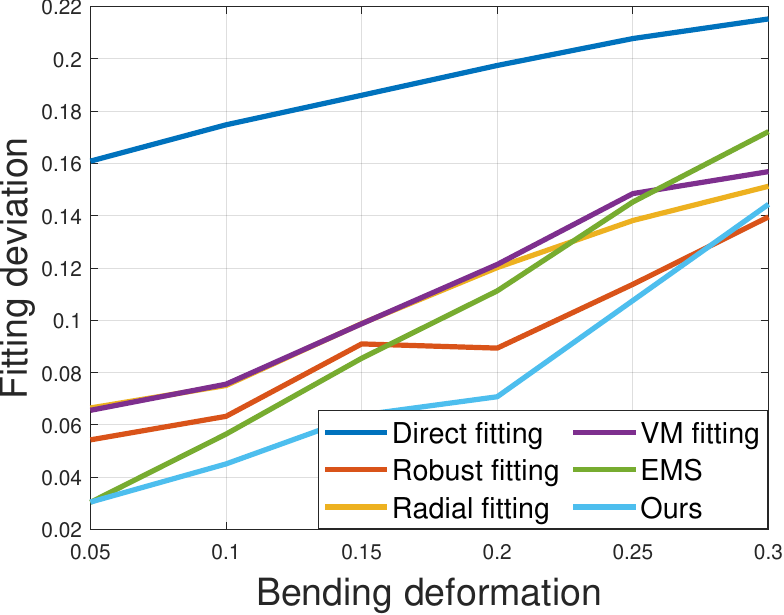}
\includegraphics[width=0.497\linewidth]{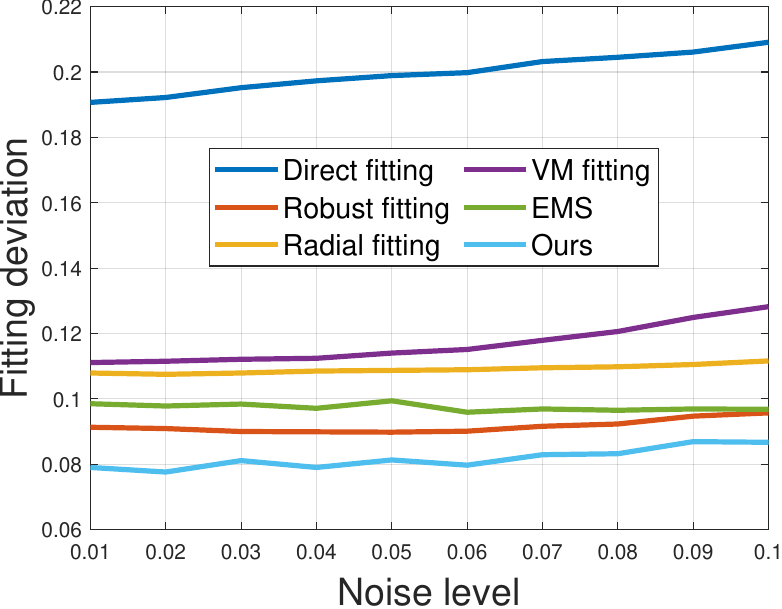}
\caption{Quantitative comparisons of bending deformable superquadric fitting (left) and an additional noise robustness test (right).} 
\label{fig:bending_test_robust}
\end{figure}
\subsubsection{Quantitative Results on Tapering Deformations} 
We first assess the proposed method under tapering deformations. Apart from the parameter setting used for rigid superquadrics (\cref{subsubsec:rigid_parameter_setting}), we further introduce the tapering parameters \( k_1 \), \( k_2\in\mathbb{R} \) to deform the superquadric surface. The tapering deformation intensity increases from $0.1$ to $1.0$ with \( k_1 = k_2 \) and a partial visible ratio of \( r = 0.5 \). For each tapering deformation, we randomly generate 100 superquadrics, resulting in a total of 1,000 test cases. As shown in the left panel of~\cref{fig:tapering_test}, compared with previous approaches, our method consistently achieves the highest shape fitting accuracy across all tapering deformations. It still maintains significantly lower fitting error even under substantial deformations (\(\eg, k_1 = k_2 = 1 \)). Moreover, the right boxplot of~\cref{fig:tapering_test} summarizes results from an additional 1,000 tests, where \(k_1, k_2\) are independently sampled from \([0, 1]\) (\ie, \((k_1, k_2) \in [0, 1]^2\), not necessarily equal). As can be seen, our method still delivers the lowest median and minimum fitting error along with the narrowest interquartile range across all tests, suggesting its superior fitting accuracy and stability. Qualitative comparison results are presented in~\cref{fig:tapering_test_vis}.

\subsubsection{Robustness Test} Next, we investigate the robustness of the proposed method against external disturbances, including noise and outliers. Gaussian noise with zero mean and increasing standard deviations 
\(\sigma \in [0.01, 0.1]\) (with 
\(\Delta\sigma = 0.01\)) is added to all the 1,000 previously generated random deformable superquadrics. The left panel of~\cref{fig:tapering_test_robust} reports the average fitting error under noise perturbations. It can be seen that our method achieves the highest parameter estimate accuracy across all noise settings and significantly outperforms competing approaches. Additionally, the right panel of~\cref{fig:tapering_test_robust} summarizes fitting results under outlier disturbances, with the outlier ratio increasing from \(0.1\) to \(0.5\) of the original point size. Our method still maintains the 
lowest fitting error, indicating its high robustness and stability.  Qualitative comparison results under noise and outlier disturbances are presented in~\cref{fig:tapering_test_robust_vis}.

\begin{table}[!htbp]
\centering
\caption{{Quantitative comparisons on the real-scanned KIT object modal dataset. $Q_1$ and $Q_2$ represent the 25th and 75th pencentiles, separately. \textbf{Bold} values stand for the top performer.}}
\vskip -0.2cm
\renewcommand{\arraystretch}{1.15} 
\begin{adjustbox}{width=\linewidth}
\begin{tabular}{l|c|c|c|c|c}
\Xhline{1pt}
{Method}&$Q_1$&{Average} &{Median} &$Q_3$&{Time (s)}\\ 
\cline{2-5}
\Xhline{1pt}
Direct~\cite{gross1988error}&3.716&4.684& 4.638&5.356&    0.978\\
{Robust~\cite{hu1995robust}}&0.681&1.946&  1.313&2.726& 1.471 \\
Radial~\cite{gross1988error}&0.631&1.213&0.947&1.618&1.976\\
VM~\cite{solina1990recovery}&0.639&1.155&0.970&1.553&0.841\\
EMS~\cite{liu2022robust}&0.622&1.369&0.994&1.833&\textbf{0.506}\\
Ours&\textbf{0.582}&\textbf{1.099}&\textbf{0.814}&\textbf{1.355}&0.792\\
\Xhline{1pt}
\end{tabular}
\end{adjustbox}
\label{tab:KIT}
\end{table}

\subsubsection{Quantitative Results on Bending Deformations} Then, we demonstrate the efficacy of the proposed method for fitting superquadrics with complex bending deformations. The shape parameters \(\epsilon_1\) and \(\epsilon_2\) are uniformly sampled from \((0,0.4]\) and \((0,2]\), respectively. The curvature parameter \(\kappa\) is ranged from \(0.05\) to \(0.3\) with a step size of \(\Delta \kappa = 0.05\) and \(\alpha\) is uniformly sampled from \((0, \pi/2]\). The scale parameters follow the same definitions as in the rigid cases. For each value of \(\kappa\), we randomly generate 100 deformable superquadrics, resulting in a total of 600 bending superquadrics. The left panel of \cref{fig:bending_test_robust} 
reports the test results under increasing degrees of bending deformations. Our method achieves superior or comparable performance across all levels of bending deformations. Moreover, to demonstrate the robustness of the developed algorithm, we further add additional Gaussian noise to the 600 randomly generated bending superquadrics. The right panel of \cref{fig:bending_test_robust} shows that our method  maintains highly robust and stable performance, consistently outperforming its competitors. Qualitative comparison results are provided in \cref{fig:bending_test_vis}.

\begin{figure*}[h]
\centering
\includegraphics[width=0.98\linewidth]{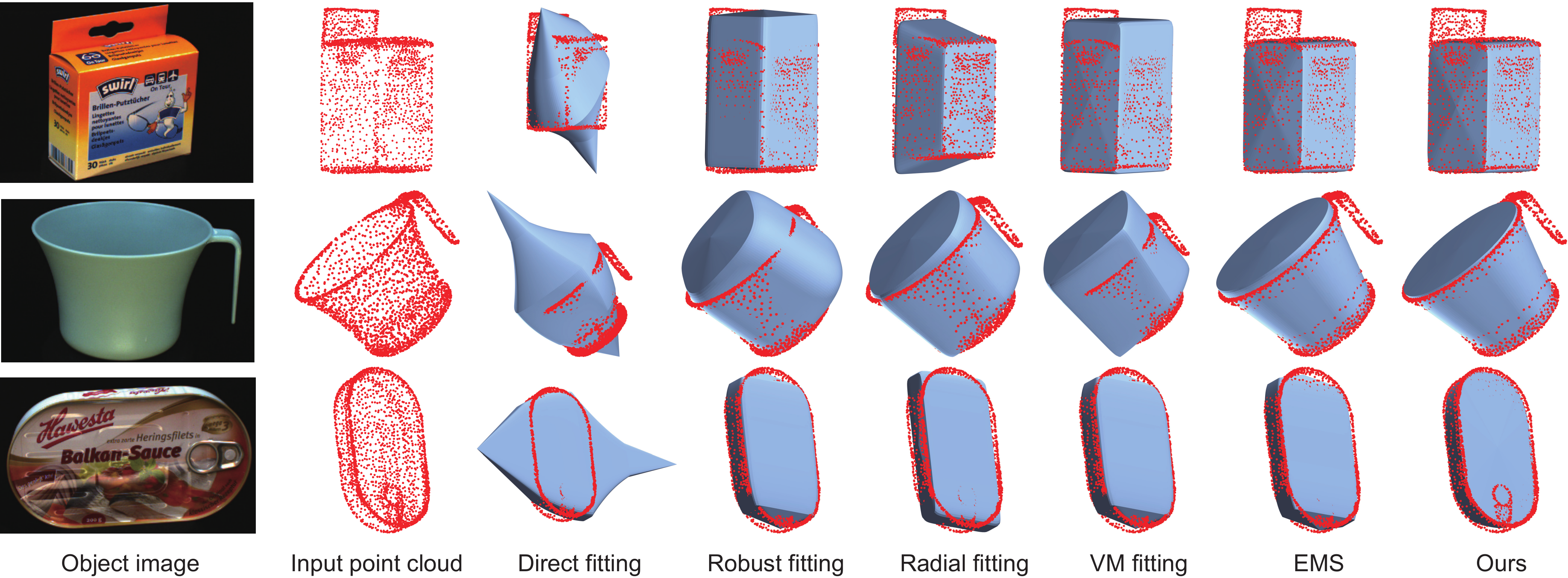}
\vskip -0.2cm
\caption{Qualitative comparisons of deformable superquadric fitting (\eg, with tapering) on real-scanned objects. In contrast to competing approaches, our method achieves superior accuracy in capturing both the \emph{dominant geometry} and \emph{fine structural details}, such as the rim of the cup and the edges of the can box.}
\label{fig:KIT_vis}
\end{figure*}

\begin{figure*}[t]
    \centering
\includegraphics[width=0.97\linewidth]{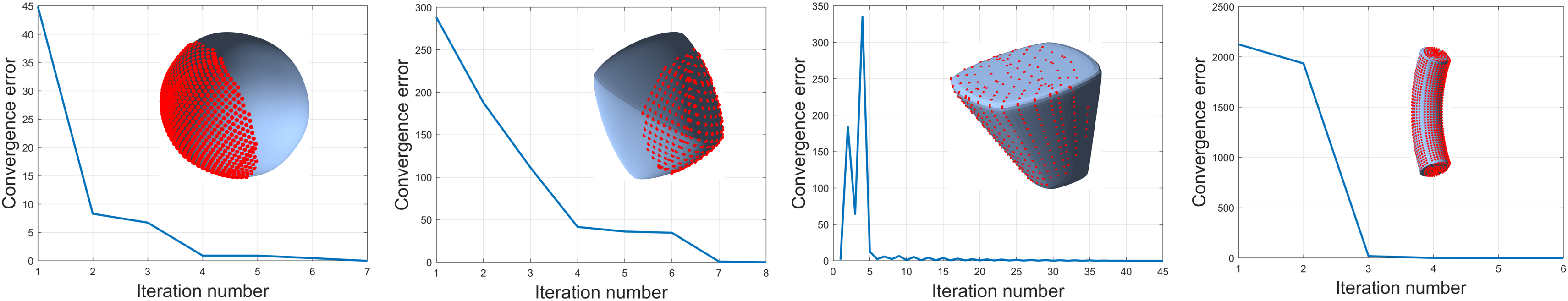}
\caption{Visualization of the convergence process for fitting rigid (first two columns), tapered (third column), and bent (last column) superquadrics. The optimization uses a convergence threshold of \(1 \times 10^{-3}\). The input point cloud is shown in red, and the fitted superquadric surface is rendered in light gray.}
\label{fig:convergence_vis}
\end{figure*}

\begin{figure}
\centering
\includegraphics[width=0.49\linewidth]{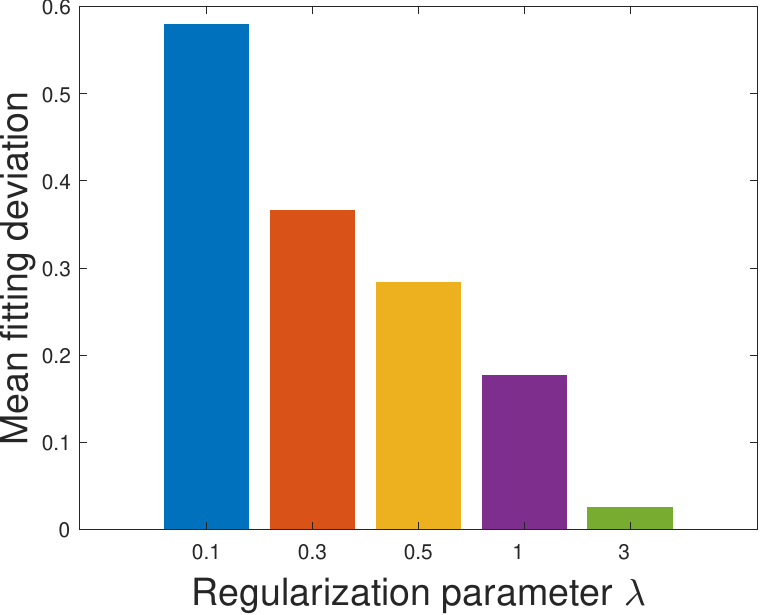}
\includegraphics[width=0.49\linewidth]{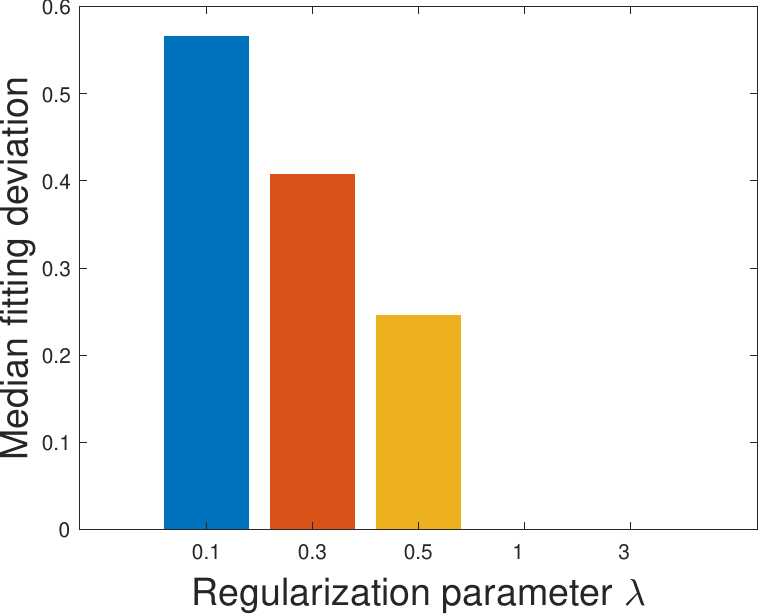}
\caption{Ablation study on the effect of the $u_i\log u_i$ regularization term (weighted by $\lambda$) for escaping local minima in superquadric fitting. The mean (left) and median (right) fitting deviations are reported.}
\vskip -0.5cm
\label{fig:lambda_ablation}
\end{figure}

\subsection{Real-Scanned Data Assessment}\label{sec:experiment_real}
Subsequently, we evaluate the proposed method on the real-scanned KIT object model database~\cite{kasper2012kit}, which contains 145 household objects designed for robotics recognition, localization, and manipulation. We manually select 100 objects that can be reasonably approximated by a single superquadric. Since most objects can be described by rigid or tapered superquadrics, we leverage superquadrics with tapering deformations for shape fitting. The statistical results in~\cref{tab:KIT} demonstrate that our method consistently achieves the highest fitting accuracy across all metrics, including the $Q_1$(25th) percentile, average, median (\ie, the $Q_2$(50th) percentile), and the $Q_3$(75th) percentile, suggesting its high stability. Moreover, while achieving remarkable accuracy, our method also maintains the second-best efficiency, making it a highly
practical and effective solution. The qualitative comparison
results in~\cref{fig:KIT_vis} indicate that our method not only provides consistently better deformable superquadric fitting or approximation to the dominated geometry of the object but also more accurately and faithfully recovers the geometric details, such as the rim region of the cup and the edges of the can box.

\begin{figure}[t]
    \centering
    \subcaptionbox{$\lambda=0.1$}{
    \begin{minipage}{0.30\linewidth}
      \centering
    \includegraphics[width=\linewidth]{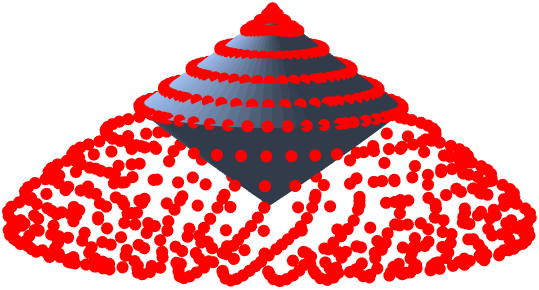}\\
    \includegraphics[width=\linewidth]{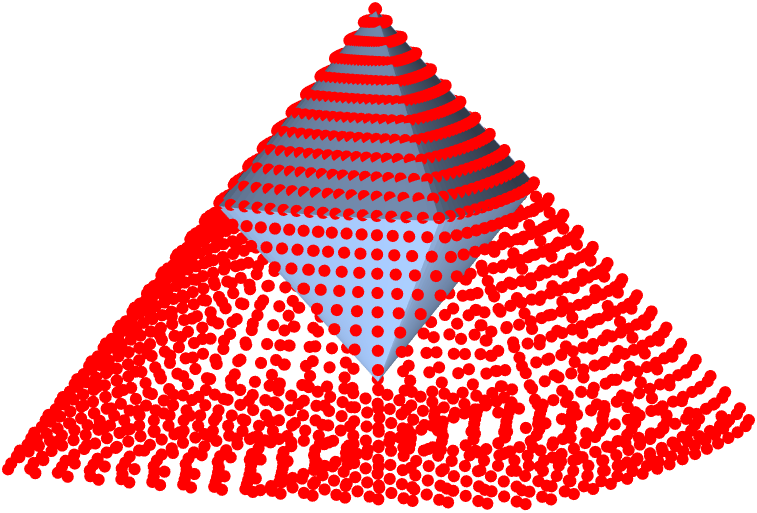}
    \end{minipage}
    }
    \subcaptionbox{$\lambda=0.5$}{
    \begin{minipage}{0.30\linewidth}
      \centering
    \includegraphics[width=\linewidth]{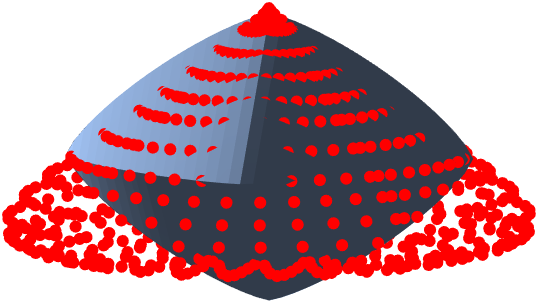}\\
    \includegraphics[width=\linewidth]{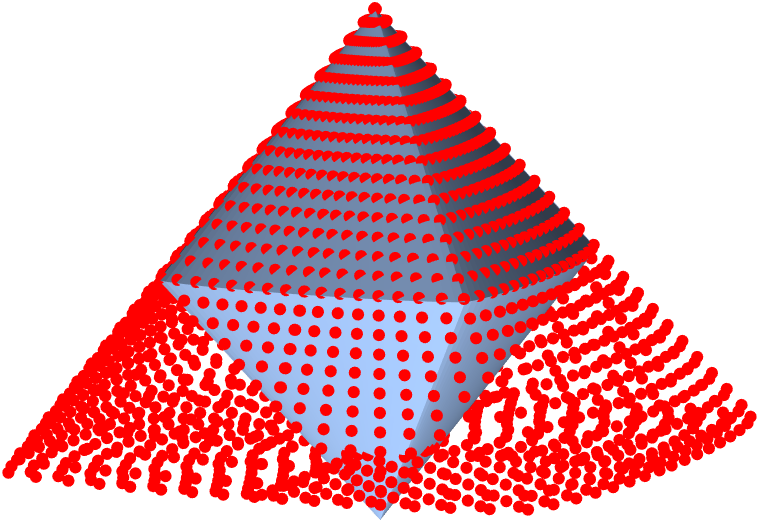}
    \end{minipage}
    }
    \subcaptionbox{$\lambda=3.0$}{
     \begin{minipage}{0.30\linewidth}
      \centering
    \includegraphics[width=\linewidth]{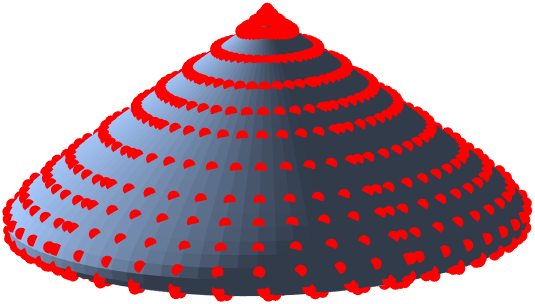}\\
    \includegraphics[width=\linewidth]{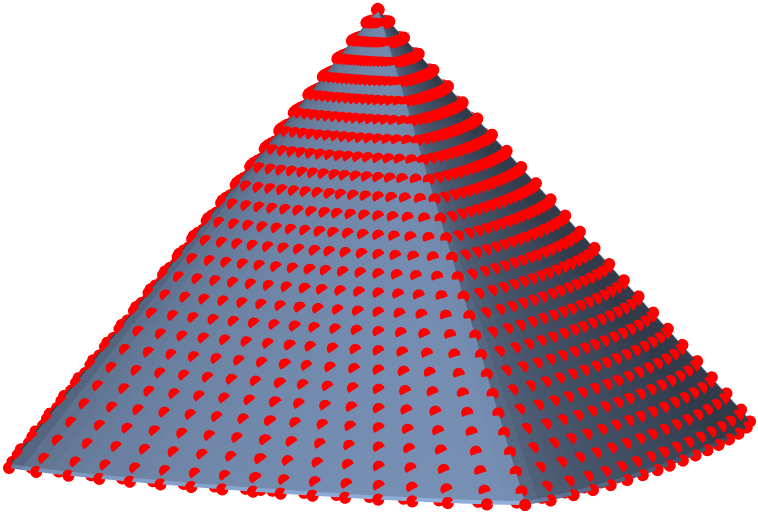}
    \end{minipage}
    }
    \caption{Illustration of the influences of the $u_i\log u_i$ regularization term (weighted by $\lambda\in\{0.1, 0.5, 3.0\}$) on highly tapered superquadric fitting.}
    \label{fig:lambda_ablation_vis}
\end{figure}
\begin{figure*}[!htbp]
\centering
\includegraphics[width=0.95\linewidth]{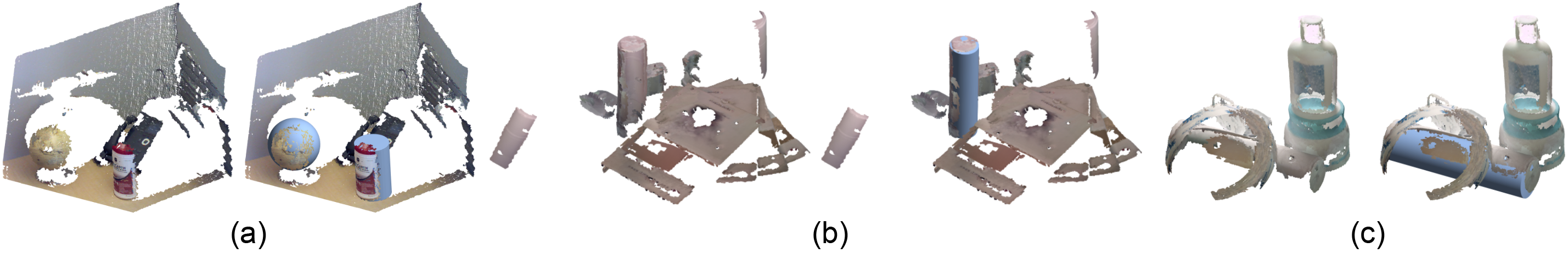}
\caption{Adaptation of the proposed method to type-specific fitting (spherical and cylindrical primitives) under heavy occlusion. The first scene is sourced from~\cite{Matlab}; the last two are from our own scans.}
\label{fig:cylinder}
\end{figure*}

\begin{figure*}[t]
\centering
\includegraphics[width=0.9\linewidth]{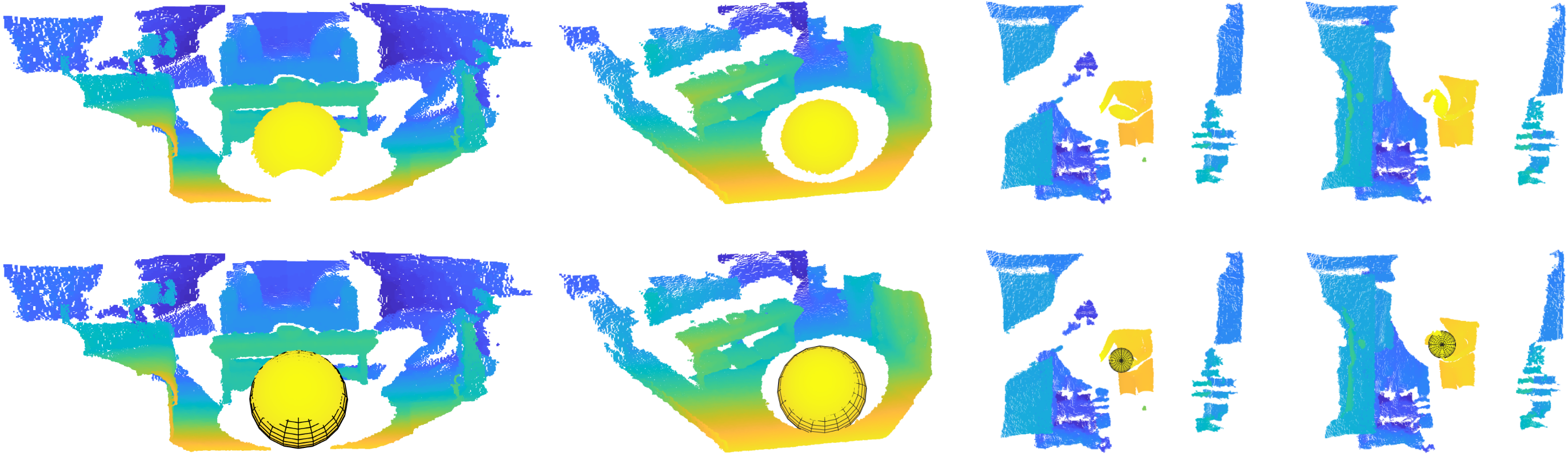}
\caption{Adaptation of the proposed method to sphere-specific primitive fitting under various occluded scenarios.}
\label{fig:sphere}
\end{figure*}

\subsection{Convergence Test}\label{sec:experiment_convergence}
The proposed method has demonstrated impressive performance in both rigid and deformable superquadric fitting. Subsequently, we investigate its convergence behavior using four randomly generated point cloud datasets: two from rigid superquadrics, one from a tapering superquadric, and one from a bending superquadric, as shown in the inset of of~\cref{fig:convergence_vis}. During fitting, we record the corresponding fitting error in each iteration and the iteration number. As can be seen from~\cref{fig:convergence_vis}, our method converges rapidly for all cases, reaching the predefined objective function threshold $1\times10^{-3}$ within just a small number of iterations, even for more challenging deformable shapes involving tapering or bending. Additionally, the final convergence errors for the four superquadrics are $4.49\times10^{-9}$, $1.47\times10^{-5}$, $2.97\times10^{-5}$, and $6.06\times10^{-4}$, respectively, indicating strong numerical stability and precision in achieving convergence.

\subsection{Effect of $u_i\log u_i$ on Avoiding Local Minima}\label{sec:experiment_ablation}
We have analyzed in~\cref{sec:advantages} that the term $u_i\log u_i$ enhances the convexity of the objective function in~\cref{eq:loss_fit_final}, where the regularization parameter \(\lambda\) 
controls the degree of convexity. This advantage can mitigate convergence to local minima, particularly for highly tapered superquadrics with \(k_1 = k_2 = \pm 1\).
To quantify this effect, we experimentally investigate how $\lambda$ influences fitting efficacy by varying $\lambda\in\{0.1, 0.3, 0.5, 1.0, 3.0\}$ on 100 randomly generated highly tapered superquadrics (larger \(\lambda\) corresponds to stronger convexity). The statistical results including mean and median fitting deviations are reported in~\cref{fig:lambda_ablation}. As observed, increasing $\lambda$ significantly reduces the fitting error, and thus facilitates \revise{escaping} from the local minima. Moreover, at $\lambda=3$, the proposed algorithm achieves the lowest mean fitting error along with a median error equal to $0$, indicating its high accuracy and stability.~\cref{fig:lambda_ablation_vis} visualizes how larger values of $\lambda$, particularly $\lambda=3$, prevent the optimization getting stuck into the local minimizer, especially for challenging highly tapered superquadrics. 

\begin{figure}
    \centering
    \subcaptionbox{}{
    \includegraphics[width=0.2\linewidth]{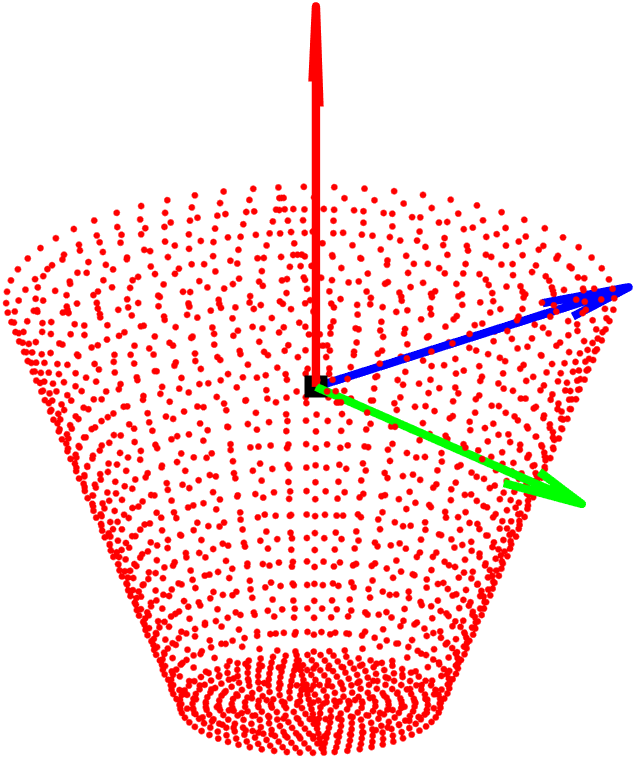}}
    \subcaptionbox{}{
    \includegraphics[width=0.2\linewidth]{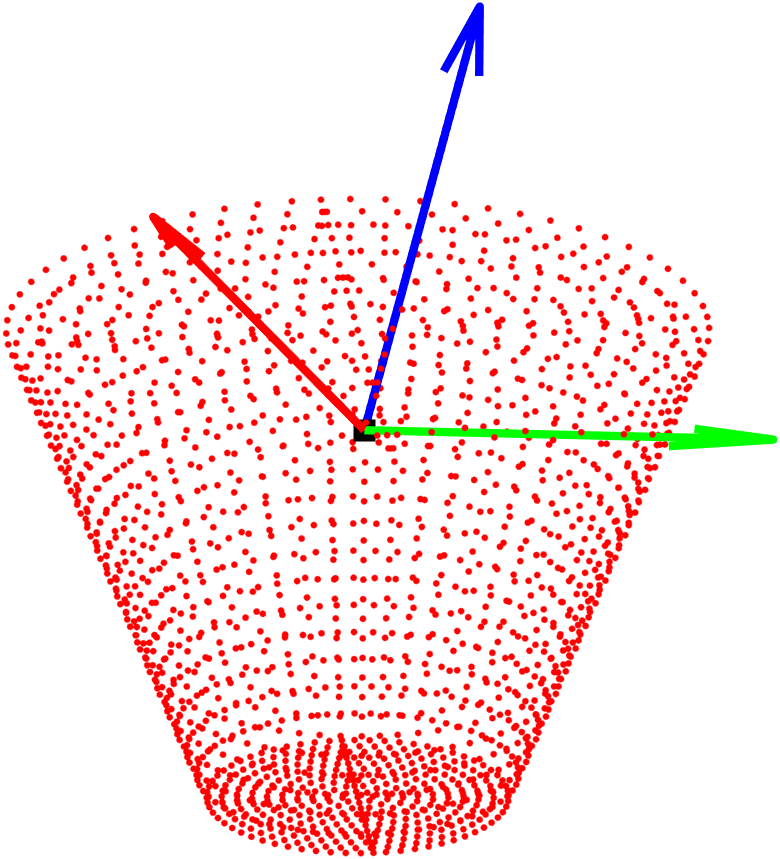}}
    \subcaptionbox{}{
    \includegraphics[width=0.2\linewidth]{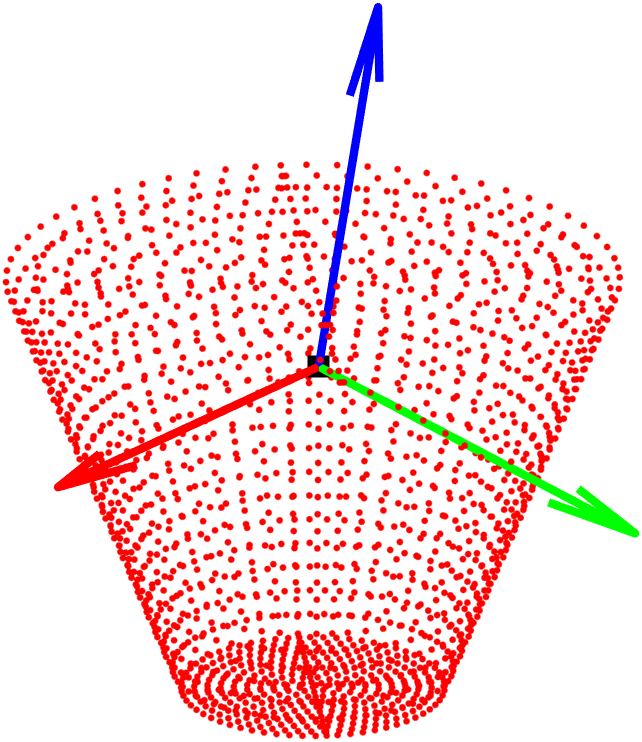}}

    \subcaptionbox{}{
    \includegraphics[width=0.2\linewidth]{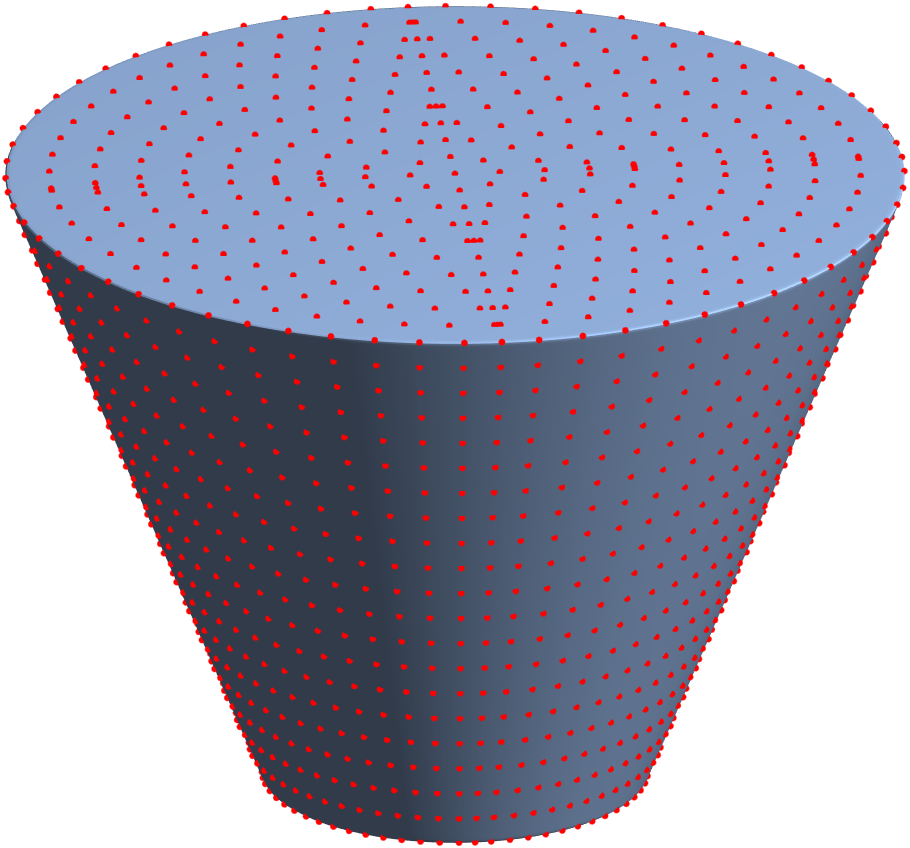}}
    \subcaptionbox{}{
    \includegraphics[width=0.2\linewidth]{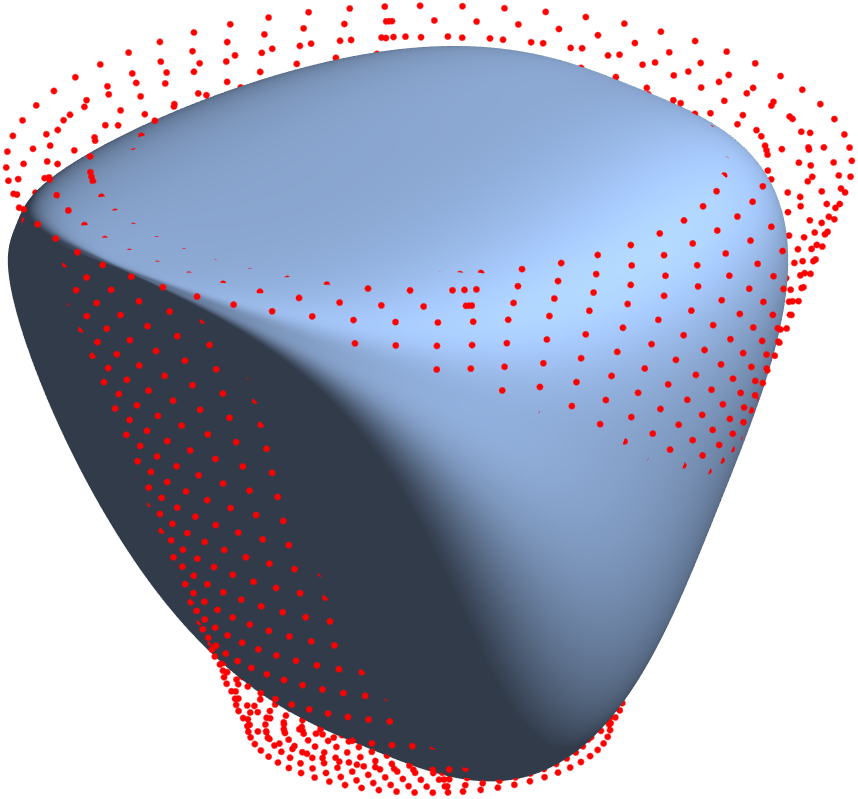}}
    \subcaptionbox{}{
    \includegraphics[width=0.2\linewidth]{Fig/PCA_90.png}}
    \subcaptionbox{}{
    \includegraphics[width=0.2\linewidth]{Fig/PCA_90.png}}

    \caption{\revise{Sensitivity analysis of PCA-based initialization. (a) PCA initialization; (b) rotation by $45^\circ$ about the $x$-axis; (c) rotation by $90^\circ$ about the $x$-axis; (d) fitting result using PCA initialization; (e) fitting result under $45^\circ$ rotation; (f) fitting result under $45^\circ$ rotation with three different axis initializations; (g) fitting result under $90^\circ$ rotation. To ensure consistently high-quality fitting, we recommend performing three initializations to avoid local minima when the PCA result deviates significantly from the correct orientation.}}
    \label{fig:PCA_init}
\end{figure}
 \revise{\subsection{Sensitivity Analysis of PCA-based Initialization}}
 {We conduct additional experiments to investigate the sensitivity of our method to PCA-based initialization. As shown in Fig.~\ref{fig:PCA_init}, with correct PCA initialization (Fig.~\ref{fig:PCA_init}(a)), our method yields a successful fit (Fig.~\ref{fig:PCA_init}(d)). Figs.~\ref{fig:PCA_init}(b) and (c) show initial orientations deviating from the correct one by $45^\circ$ and $90^\circ$, respectively, via rotation about the $x$-axis. Fig.~\ref{fig:PCA_init}(g) indicates that our method can still achieve successful fitting when the deviation is $90^\circ$; however, it returns an incorrect result when the deviation is $45^\circ$. To address this issue, we perform three initializations using different axes as the $x$-axis and select the fitting result with the lowest error. Fig.~\ref{fig:PCA_init}(f) demonstrates that our method successfully recovers the original superquadric. Therefore, to ensure consistently high-quality fitting, we recommend performing three initializations to avoid local minima when the PCA result deviates significantly from the correct orientation. }

\subsection{Adaptation to Type-Specific Fitting} 
\label{sec:type_specific_fitting}
While the proposed method is designed for general (\ie, type-agnostic) superquadric fitting, it can be readily adapted to fit \emph{type-specific primitives} for specialized applications. To validate its adaptability, we apply our algorithm to two common cases: cylinder fitting and  sphere fitting.

Focusing first on cylinder fitting, by setting $\epsilon_1 = 0$ and $\epsilon_2 = 1$ in \cref{eq:parametric}, the equation transforms to:  
\begin{equation}
\bm{x}(\eta, \omega) = 
\left[
\begin{array}{c}
a_{x} C_{\omega}^{\varepsilon_{2}} \\
a_{y} S_{\omega}^{\varepsilon_{2}} \\
a_{z} 
\end{array}
\right] = 
\left[
\begin{array}{c}
a_{x} \cos(\omega) \\
a_{y} \sin(\omega) \\
a_{z} 
\end{array}
\right], 
\label{eq:cylinder}
\end{equation} which degenerates to a canonical cylindrical parametric form. To strictly enforce the fitting result as a cylinder-specific surface (with uniform circular cross-section), we compute the average of $a_x$ and $a_y$ and use this value as the cylinder radius.  

For spherical surface fitting, we similarly tailor the parameters: setting $\epsilon_1 = \epsilon_2 = 1$ in \cref{eq:parametric}, and averaging $a_x$, $a_y$, and $a_z$ to define the spherical radius $r$. To quantify the spherical fitting performance, we adopt the point-to-sphere distance metric:  
\begin{equation}
\text{dist}(\bm{x}, \bm{c}) = \left| \|\bm{x} - \bm{c}\|_2 - r \right|,
\label{eq:sphere_distance}
\end{equation} where $\bm{x}\in\mathbb{R}^3$ denotes a target point and $\bm{c}\in\mathbb{R}^3$ denotes the sphere center. For cylinder fitting, we retain \cref{eq:metric} for evaluation, as a closed-form analytical solution for the point-to-cylindrical surface distance remains unavailable.

\begin{table}[t]
\renewcommand{\arraystretch}{1.5}
\caption{Quantitative comparisons of the type-specific fittings between our method and the algorithm~\cite{torr2000mlesac} built in MATLAB.
 }
\centering
\begin{tabular}{c|cc|cc}
    \hline
  \multirow{2}{*}{\diagbox{Primitive}{Method}}&\multicolumn{2}{c|}{\cite{torr2000mlesac}} &\multicolumn{2}{c}{Ours}\\
  \cline{2-5}
    &Average & Median & Average & Median \\ 
   \hline
    Sphere&0.0041 & 0.0043 & 0.0072 & 0.0061 \\
    \cline{1-5}
    Cylinder&2.806 & 3.314 & 0.781 & 1.100 \\
    \hline
\end{tabular}
\label{tab:sphere_cylinder_fitting}
\end{table}

\begin{figure*}[t]
    \centering
\includegraphics[width=0.98\linewidth]{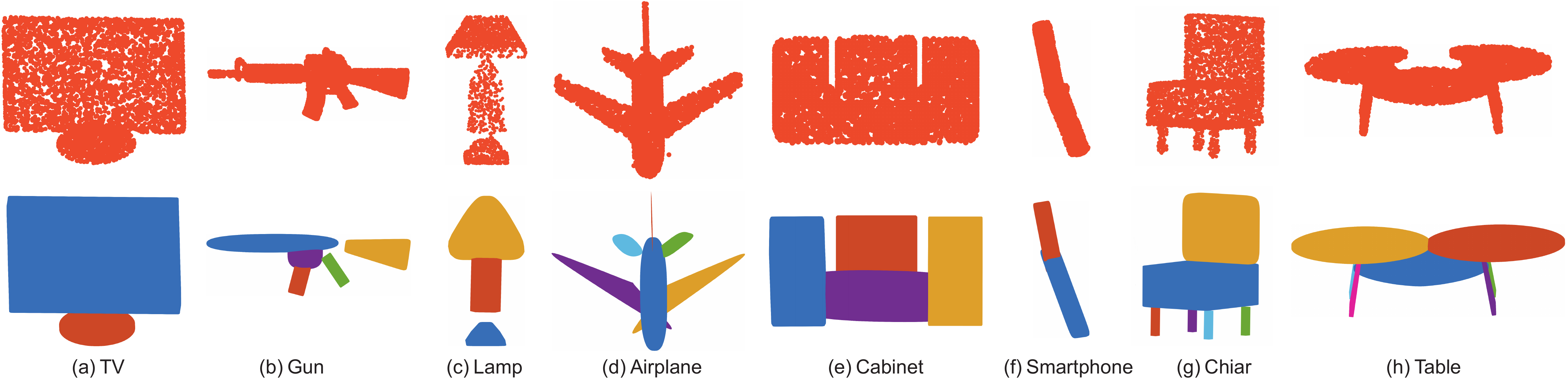}
\caption{Qualitative results on the shape \revise{approximation} for objects from the ShapeNet dataset~\cite{chang2015shapenet} by 
superquadrics. Both rigid and deformable (\eg, the wing of an airplane and the buttstock of a gun) superquadrics are accurately fitted.} 
\label{fig:shapenet_vis1}
\vskip -0.3cm
\end{figure*}

We evaluate the proposed algorithm using three distinct scene datasets. The dataset in \cref{fig:cylinder}(a) is provided by MATLAB~\cite{Matlab} and contains both spherical and cylindrical primitives, whereas those in \cref{fig:cylinder}(b) and (c) were scanned by ourselves using a Revopoint Mini2 3D scanner, containing cylindrical primitives that are occluded by foreground objects (\eg, books). The datasets in \cref{fig:sphere} are sourced from \cite{choi2016large} and include spherical primitives across diverse scene configurations. Notably, all datasets exhibit severe occlusions and partial object visibility, arising from the \emph{Field-of-View} (FOV) constraints of the employed sensors. We compare our method with  the MATLAB-implemented version of the algorithm in \cite{torr2000mlesac} (a widely used baseline for robust fitting). \cref{tab:sphere_cylinder_fitting} presents the average and median fitting accuracy across the aforementioned datasets. Under severe occlusions, our method, despite not being specifically tailored for spherical fitting, still yields high-quality results comparable to the MATLAB-implemented baseline. Moreover, it achieves superior fitting accuracy for cylindrical primitives compared to \cite{torr2000mlesac}, which underscores its robustness to occlusions and versatility across primitive types. Visualizations of the fitted spherical and cylindrical surfaces overlaid on the original point clouds are provided in \cref{fig:cylinder} and \cref{fig:sphere}, respectively.

\subsection{Applications}\label{sec:experiment_applications}
After showing the accuracy and robustness of our algorithm, in this section, we further demonstrate its versatility by applying the proposed superquadric fitting method to three downstream tasks: 
{\emph{shape approximation}}, \emph{geometry editing}, and \emph{medical modeling}. For shape \revise{approximation}, we integrate our fitting method into both traditional optimization and learning-based frameworks, highlighting its adaptability across diverse methodological paradigms. For geometry editing, we leverage the fitted parameters to manipulate the underlying geometry, enabling the generation of novel shapes. For medical modeling, we showcase the utility of deformable bending superquadrics in \revise{characterizing blood vessel structures}.

\subsubsection{\revise{Shape Approximation}}
Decomposing a complex shape into a set of geometric primitives facilitates enhanced perceptual understanding aligned with the human visual system~\cite{biederman1987recognition,edelman1999representation,kinzler2007core}. Consequently, this decomposition strategy emerges as a promising pathway for intelligent systems to accomplish high-level objectives, such as environmental planning and interactive tasks, by leveraging structured shape parsing~\cite{liu2022robust,fedele2025superdec}. Benefiting from the compact and expressive representation capacity of superquadrics,  we integrate our fitting algorithm to {approximate and model} the constituent parts from an initial coarse segmentation. We demonstrate this versatility using segmentation derived from both traditional optimization and learning-based frameworks.\\

{\textit{1) Integration with Traditional Optimization Framework.}}
Traditional optimization-based segmentation approaches typically adopt an unsupervised clustering paradigm. To this end, we integrate our method with a recent clustering-based point cloud segmentation technique~\cite{wu2022primitive} that roughly partitions an object into distinct regions. As illustrated in~\cref{fig:shapenet_vis1}, we evaluate eight diverse object categories from the ShapeNet dataset~\cite{chang2015shapenet} to validate the generalization capability of our proposed method. Specifically, we first apply the segmentation framework from~\cite{wu2022primitive} to obtain a coarse partition of the object, followed by invoking our superquadric fitting algorithm.

The shape decomposition results presented in~\cref{fig:shapenet_vis1} demonstrate that our method not only faithfully captures the overall geometric structures of the objects, enabling more compact storage and description compared to discrete and unorganized point clouds, but also accurately recovers \emph{geometric deformations} such as tapering structures, which are prevalent in various objects (\eg, the wing of an airplane and the buttstock of a gun).\\

\begin{table}[t]
\caption{Quantitative comparisons with state-of-the-art methods on the ShapeNet test dataset, which comprises 8,751 models.}
\label{tab:shapenet_part}
\centering
\resizebox{\linewidth}{!}{
\begin{tabular}{l l c | c c c }
\toprule
\multirow{2}{*}{\textbf{Model}} & \multirow{2}{*}{\textbf{Primitive Type}} & \multirow{2}{*}{\textbf{Segmentation}} & \multicolumn{3}{c}{\textbf{ShapeNet test dataset}}  \\
& & & {${L_1}$ $\downarrow$} & {${L_2}$ $\downarrow$} & {\# Prim.$\downarrow$}  \\
\midrule
EMS~\cite{liu2022robust} & Superquadrics & \XSolidBrush & 5.771 & 1.345 & 5.68  \\
CSA~\cite{yang2021unsupervised} & Cuboids & \Checkmark & 5.157 & 0.527 & 9.21  \\
SQ~\cite{paschalidou2019superquadrics} & Superquadrics & \XSolidBrush & 3.668 & 0.279 & 10 \\
SuperDec~\cite{fedele2025superdec} & Superquadrics &\Checkmark & {1.698} & {0.051} & {5.8}  \\
Ours & Superquadrics & \Checkmark & \textbf{1.346} & \textbf{0.035} & \textbf{5.74}  \\
\bottomrule
\end{tabular}
}
\end{table}

\begin{figure*}[t]
    \centering
    \includegraphics[width=\linewidth]{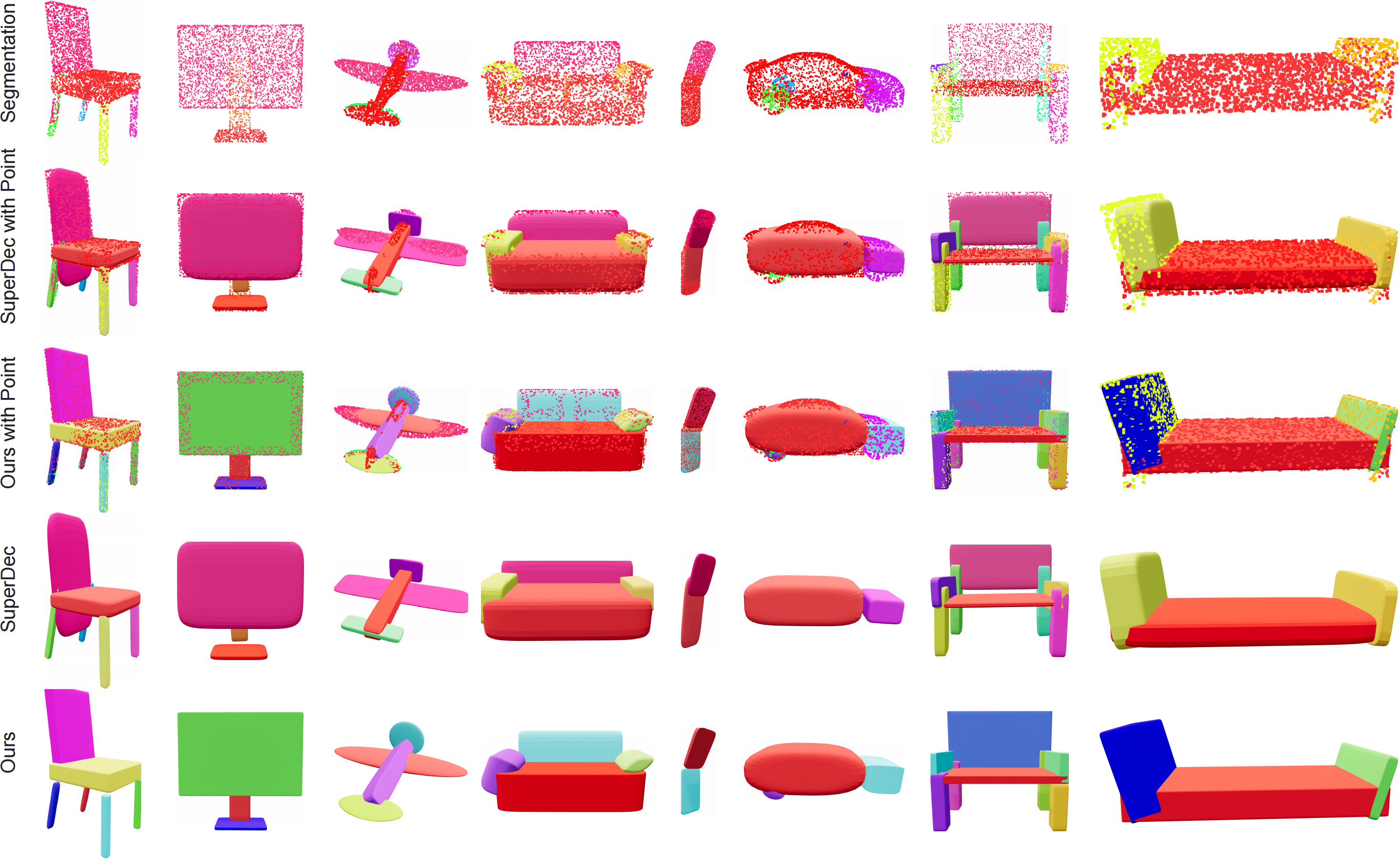}
    \vskip -0.2cm
\caption{Qualitative comparisons with SuperDec on the ShapeNet dataset. The first row presents the initial segmentation results, where different colors indicate different parts. The second and third rows show the fitted superquadrics aligned with the input point clouds; the proposed method yields a more accurate fit, as evidenced by the reduced distance between the recovered surface and the input points. The final two rows display the shape decomposition results, where our fitted superquadrics provide a \emph{tighter approximation} to the original geometry and more faithfully capture \emph{sharp features and contours}.}
\label{fig:shapenet_vis}
\end{figure*}

{{\textit{2) Integration with Learning Framework.}}} 
Next, we integrate the proposed superquadric fitting algorithm with a recent learning-based shape decomposition method SuperDec~\cite{fedele2025superdec}. SuperDec leverages a transformer-based neural network to segment point clouds and predict a set of initial superquadric parameters, which are then further optimized via the Levenberg-Marquardt (LM) method~\cite{more2006levenberg}. To demonstrate the effectiveness and versatility of our fitting approach, we directly deploy it to the segmentation outputs of SuperDec and bypass the parameter initialization step. Consistent with the experimental setup in SuperDec, we utilize its pre-trained model on the ShapeNet dataset and evaluate on the ShapeNet test set, which comprises 8,751 models. For performance benchmarking, we compare our final fitting results against four representative methods: EMS~\cite{liu2022robust}, CSA~\cite{yang2021unsupervised}, SQ~\cite{paschalidou2019superquadrics}, and SuperDec~\cite{fedele2025superdec}. Note that all learning-based methods are jointly trained on the 13 classes of the ShapeNet training set and evaluated on the corresponding test split, ensuring a fair comparison.

We report the quantitative comparison results in~\cref{tab:shapenet_part}, where the first four-row results are from~\cite{fedele2025superdec} for direct reference. As evident from~\cref{tab:shapenet_part}, our method achieves a significant reduction in shape approximation error across both {$L_1$} and $L_2$ point to point loss metrics, indicating its superior fitting accuracy and robustness. Moreover, since we adopt the segmentation results (and thus the primitive count) initialized by SuperDec, our method retains similar number of primitives for shape representation, highlighting its strong potential for compact geometric modeling without sacrificing accuracy. Qualitative decomposition results are presented in~\cref{fig:shapenet_vis}. Our method yields a much \emph{tighter approximation} to the original point cloud structures and more faithfully captures sharp features and contours, resulting in high-quality shape decompositions. {This is particularly beneficial for applications such as accurate collision detection in robot path planning and grasp pose generation.}\\

{\textit{3) Type-Specific Shape \revise{Approximation}.}
Except for shape \revise{approximation} using general superquadrics, as discussed in \cref{sec:type_specific_fitting}, our proposed fitting method can be readily adapted for \emph{type-specific} \revise{approximation}, %decomposition, 
which targets specialized application scenarios. For instance, by fixing the shape parameters to \(\varepsilon_1 = \varepsilon_2 = 1\) in~\cref{eq:implicit}, we instantiate the ellipsoid primitive: a \emph{low-algebraic-degree} quadratic surface representation widely adopted in robotics tasks, such as collision detection~\cite{choi2006continuous,lee2017velocity} and object localization~\cite{zhao2024bayesian}. To validate this type-specific adaptation, {we perform initial segmentation using~\cite{wu2022primitive} and then use ellipsoid primitives for shape approximation} on real-scanned human body point clouds from the D-FAUST dataset~\cite{CVPR2014}, \revise{which exhibit freeform geometry and inherent scanning noise.} As shown in \cref{fig:faust_vis}, our method achieves highly accurate approximation for individual anatomical part of the human body, thereby simplifying the characterization of attributes such as human pose.

\begin{figure*}[!htbp]
    \centering
\includegraphics[width=0.88\linewidth]{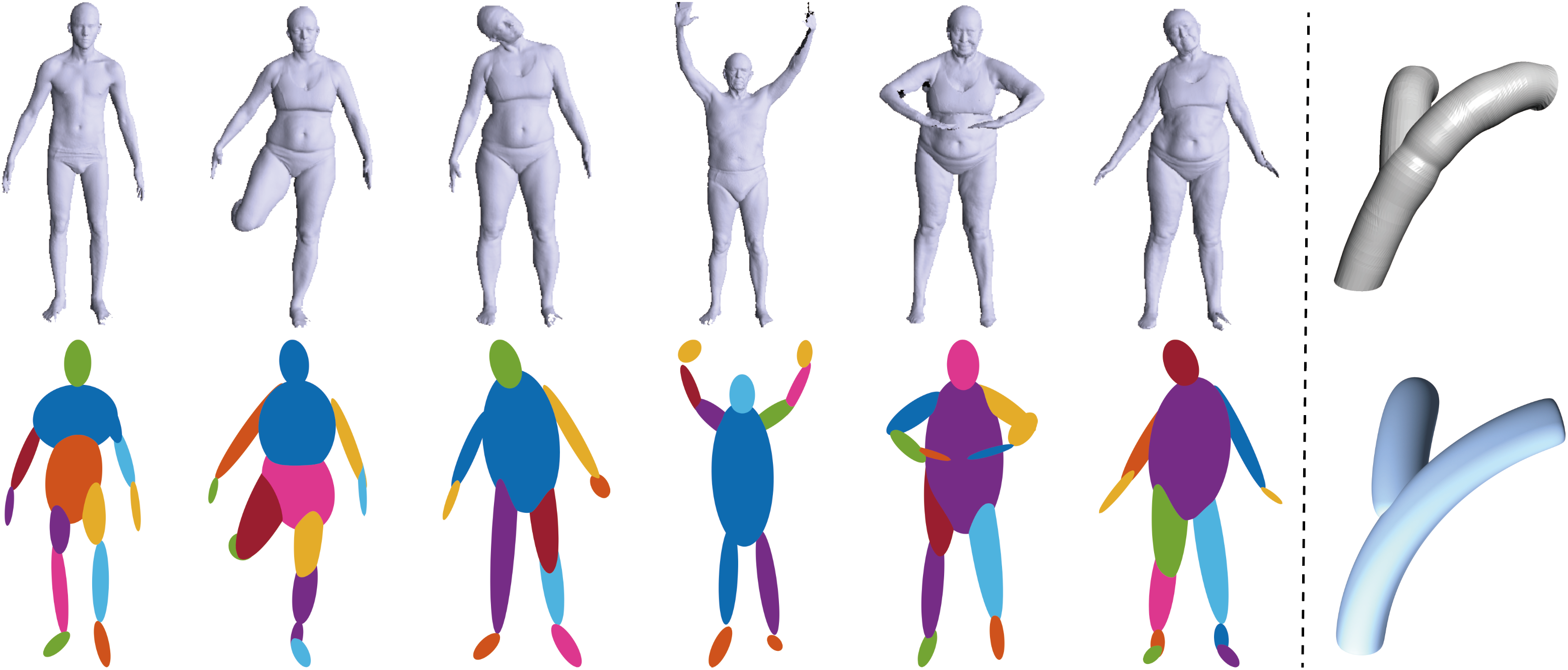}
\vskip -0.2cm
\caption{Shape \revise{approximation} using type-specific superquadrics, \ie, ellipsoid primitives (left);  application of the proposed method to medical modeling using bent superquadrics (right).} 
\label{fig:faust_vis}
\end{figure*}

\begin{figure}[!htbp]
\centering
\includegraphics[width=\linewidth]{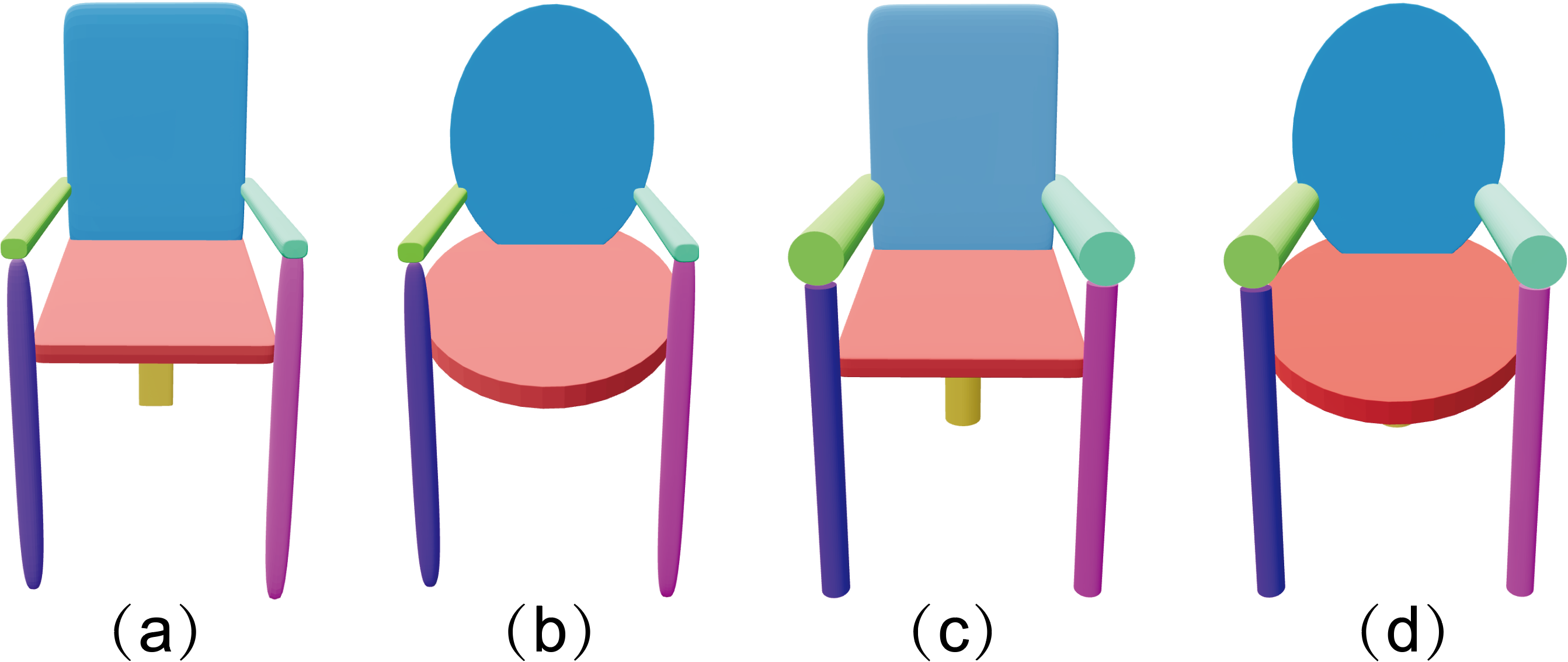}
\vskip -0.2cm
\caption{Geometry editing based on our proposed superquadric fitting algorithm. (a) The original decomposition object. (b)-(d) The newly edited results.}
\label{fig:geometry_edit}
\end{figure}

\subsubsection{Geometry Editing}
As our method provides an accurate \emph{parameterized representation} for each part of an object, the geometry can be easily edited to generate novel shapes. As presented in \cref{fig:geometry_edit}, (a) shows the original shape decomposition result of an chair. In (b), we modify the chair back and seat cushion by replacing the original superquadric primitives with cylinders. In (c), we edit the chair arms and legs. Finally, (d) presents a new object formed by combining all previous edits.

\subsubsection{Medical Modeling} 
Leveraging the expressive power of deformable superquadrics, we further demonstrate the versatility of our method by applying it to medical modeling. The right panel of \cref{fig:faust_vis} presents the fitting results of bent superquadrics (bottom-right) for the blood vessel model (top-right) sourced from the 3D GrabCAD repository~\cite{GrabCAD}.
{This enables joint estimation of boundaries, centerlines, and local poses in vascular images~\cite{tyrrell2007robust}. Another key advantage of using the superquadric model is that its local nature allows us to adapt statistics related to vessel appearance across different regions of the vasculature.}
As shown, our method consistently and accurately reconstructs the blood vessel geometry using bending deformations, which is typically challenging to approximate using general or rigid primitives.

\section{Limitations}\label{sec:limitation}
In this section, we summarize the limitations of the proposed superquadric fitting method.
\subsection{Unification of Segmentation and Fitting}
While our method has demonstrated promising fitting performance in terms of accuracy, efficiency, and robustness for both rigid and deformable superquadrics, its current focus is limited to the precise optimization or accurate estimation of parameter sets for a \revise{single} superquadric. For complex objects composed of multiple components, segmentation algorithms are typically indispensable, {as structured noise (\eg, a floor plane beneath an object) often introduces significant bias into parameter initialization, particularly for the position $\mathbf{t}$ and the rotation matrix $\mathbf{R}$, thereby affecting the fitting result. We provide an experimental analysis of structured noise in the \emph{Supplementary Material}.}

Given that our method is rooted in unsupervised clustering analysis, and clustering analysis itself can be applied to shape segmentation, a promising future direction lies in developing a \emph{unified} superquadric fitting system. This system would be capable of handling complex objects by first performing segmentation and then conducting parameter fitting, \emph{all within the same fuzzy clustering framework}. Such a unification can significantly enhance the understanding of shape modeling, entirely from the perspective of clustering.

\subsection{Determination of the Optimal Parameter $\lambda$}
Another limitation of our method is the lack of an automated technique to efficiently determine the optimal value of \(\lambda\), which controls the degree of convexity of the loss function in~\cref{eq:loss_fit_final} and helps avoid local minima. {Typically, larger deformations require increasing the regularization parameter $\lambda$, as a larger $\lambda$ enhances convexity but may come at the cost of reduced accuracy.} Determining the optimal \(\lambda\) for each fitting task remains a non-trivial and time-consuming process. Nevertheless, our empirical results across both synthetic and real-world settings indicate that setting \(\lambda=3\) generally yields stable and accurate fittings without the need for additional manual tuning.

\subsection{Type-Agnostic Fitting}
 The proposed algorithm effectively fits both rigid and deformable superquadrics. Nevertheless, it requires the deformation type (\eg, tapering or bending) to be specified in advance, as this choice influences the parameter initialization. {To circumvent this limitation, directly fitting a raw, unorganized point cloud without prior structural knowledge entails a trial-and-error process to determine the most suitable deformation type. Alternatively, one can automatically select the deformation type based on model complexity analysis and assessment~\cite{hastie2009elements}, where more complex deformations correspond to more complex fitting models. This can be achieved using statistical criteria such as the Minimum Description Length (MDL), Akaike Information Criterion (AIC), or Bayesian Information Criterion (BIC), with the best parameters chosen by minimizing the selected criterion.}

\section{Conclusions and Future Directions}\label{sec:conclusions}
We have presented a novel algorithm for the accurate and robust fitting of superquadrics, unifying the treatment of \emph{rigid} and \emph{deformable} models within a \emph{holistic} optimization framework. Our core contribution is the reformulation of the fitting problem as an \emph{unsupervised clustering process}. This novel perspective enhances the convexity of the objective function, thereby facilitating optimization and mitigating the problem of local minima, as validated by our experiments. We further derive the correlation between pairwise computations of clustering centroids and clustering samples with geometric distances, which circumvents the computational inefficiency of sampling from potential surfaces. Moreover, our optimization framework enables closed-form solutions for both the fuzzy membership matrix and the covariance matrix, thus ensuring a fast and efficient implementation. We also present a theoretical analysis to guarantee the algorithm's convergence.

We conduct comprehensive numerical studies to investigate the performance of our proposed  optimization and fitting framework. Extensive experiments on both synthetic and real-world data demonstrate the superior accuracy, robustness, and stability of our method. Notably, it exhibits a strong capability to \emph{avoid local minima} and to accurately fit highly deformable superquadrics. Moreover, we highlight its versatility through a series of downstream applications ranging from complex shape \revise{approximation}, geometry editing, to medical modeling.

There are several avenues exist to further extend the proposed method. One of the directions is the integration of advanced semantic segmentation techniques to enhance shape decomposition, particularly for highly \emph{free-form} surfaces. Inspired by CAD modeling, superquadrics can be combined with \emph{constructive solid geometry} (CSG) and \emph{Boolean operations} to construct more intricate models, thereby further expanding their expressive power. A similar modeling framework can involve the composition of multiple simultaneous deformations, such as tapering and bending, to represent more sophisticated shapes. Another interesting direction is to \emph{generalize Gaussian splatting} by replacing its primitive ellipsoids with more compact and expressive superquadrics, potentially unlocking new applications in rendering and reconstruction.

{\section*{Acknowledgment}
The author would like to deeply thank the anonymous referees who provided helpful and constructive comments. This work was supported by the Key Project of the National Natural Science Foundation of China (12494550 and 12494551), the Strategic Priority Research Program of the Chinese Academy of Sciences (XDB0640000, XDB0640200, and XDA0480502), and the National Natural Science Foundation of China (52505001).}

\ifCLASSOPTIONcaptionsoff
  \newpage
\fi



%\bibliographystyle{IEEEtran}


% \end{thebibliography}

\bibliographystyle{ieeetr}
\bibliography{main}%references

\newpage

\section*{Supplementary Material}

\begin{abstract}
In this supplementary material, we provide additional content to support our paper. Concretely, the following contents are presented. The implementation details and parameter settings of all compared approaches in the experimental part are reported in \cref{sec:implementation}, followed by the derivation of the closed-form solution of the fuzzy membership matrix $\mathbf{U}$ in \cref{sec:closed_form}. In~\cref{sec:1}, we present the theoretical proof of Proposition 1 in the paper. In \cref{sec:transformation}, we present the comprehensive definition of the three types of spatial inverse mappings, \ie, rigid, tapering, and bending transformations, that revert a general positional superquadric to its canonical form. Finally, we present a detailed description of the optimization bound setting for superquadric parameters in~\cref{sec:2}.
\end{abstract}

\section{Implementation Details and Parameter Settings of Compared Approaches}\label{sec:implementation}
We report the detailed implementation details and 
parameter configurations for all compared approaches, including NS fitting~\cite{vaskevicius2017revisiting}, Robust fitting~\cite{hu1995robust}, Radial fitting~\cite{gross1988error}, VM fitting~\cite{solina1990recovery}, and EMS~\cite{liu2022robust} as follows.\\

$\bullet$ NS fitting. For the numerically stable superquadric fitting algorithm proposed in~\cite{vaskevicius2017revisiting}, we adopt the parameter set recommended by the authors for constructing auxiliary functions. Detailed approximation parameters are provided in~\cref{tab:approximation}.

%In numerically stable superquadric fitting algorithm~\cite{vaskevicius2017revisiting}, we adopt the set of parameters suggested by the authors for the construction of the auxiliary functions. Concrete approximation parameters are reported in~\cref{tab:approximation}.
\begin{table}[!htbp]
\centering
\caption{Approximation parameters for the  auxiliary function construction.}
\begin{tabular}{llllllll}
\hline
$\alpha_\phi$ & $10^{-4}$ & &  $\alpha_\psi$ & 0.1 & & $\alpha_h$ &$10^{-7}$\\
$\alpha_l$ & $10^{-12}$ & & $\beta_\psi$ & 0.5 &  & & \\
\hline
\end{tabular}
\label{tab:approximation}
\end{table}

$\bullet$ Robust fitting. To circumvent the influence of non-Gaussian noise and outliers, Hu et al.~\cite{hu1995robust} developed an adaptive weighted partial-data minimization algorithm for superquadric fitting, whose mathematical formulation is given as follows:

%The authors in~\cite{hu1995robust} developed an adaptive weighted partial data minimization algorithm for superquadric fitting, which is mathematically defined as follows:
\begin{equation}
\min _{\hat{\mathbf{\Theta}}} \mathbf{O}(\mathbf{\Theta})=\sum_{i=1}^n\left(w_i R_i\right)^2,   
\end{equation}
subject to:
\begin{equation}
\sum_{i=1}^n w_i \geq \alpha n,   
\end{equation}where $w_i\in[0,1]$ is the weight factor and $\alpha\in[0.5, 1]$. $R_i$ is the residual of point $\bm{x}_i\in\mathbb{R}^3$ defined using the following function 
\begin{equation}
 R_i=\sqrt{a_x a_y a_z}\left(F^{\epsilon_1}\left(\mathbf{\Theta}\right)-1\right), 
\end{equation} where  $F\left(\mathbf{\Theta}\right)$ is the implicit function under general positions. \\

\iffalse
\begin{equation}
F(x, y, z)\!=\!\left(\left(\frac{x}{a_{x}}\right)^{\frac{2} {\epsilon_{2}}}+\left(\frac{y}{a_{y}}\right)^{\frac{2}{ \epsilon_{2}}}\right)^{\frac{\epsilon_{2}}{\epsilon_{1}}}
\!+\!\left(\frac{z}{a_{z}}\right)^{\frac{2}{\epsilon_{1}}}.
\label{eq:implicit}
\end{equation}
under general positions.  
\fi

$\bullet$ Radial fitting. The objective function for radial fitting under canonical form is defined as 
\begin{small}
\begin{equation}
\min _{\hat{\mathbf{\Theta}}}\sum_{i=1}^n\|\bm{x}_i\|\left[\left(\left(\left(\frac{x_i}{a_{x}}\right)^{\frac{2} {\epsilon_{2}}}+\left(\frac{y_i}{a_{y}}\right)^{\frac{2}{ \epsilon_{2}}}\right)^{\frac{\epsilon_{2}}{\epsilon_{1}}}
\!+\!\left(\frac{z_i}{a_{z}}\right)^{\frac{2}{\epsilon_{1}}}\right)^{-\frac{\epsilon_1}{2}}-1\right],
\label{eq:implicit}
\end{equation}
\end{small}where a revised exponent of $-\frac{\epsilon_1}{2}$ is introduced to enhance the fitting accuracy of the original implicit function.\\

$\bullet$ VM fitting. Analogous to radial fitting, VM fitting modifies the original implicit function by incorporating a minimal volume constraint, \ie, 

%Similar to radial fitting, VM fitting corrects the original implicit function by introducing minimal volume constraint, \ie, 
\begin{equation}
\min _{\hat{\mathbf{\Theta}}}\frac{1}{n}\sum_{i=1}^{n}\sqrt{a_x a_y a_z}\left(F^{\epsilon_1}\left(\mathbf{\Theta}\right)-1\right)^2.  
\end{equation} 

$\bullet$ EMS fitting. For the implementation of EMS fitting, we adopt the default parameters for test, as they yield stable performance across diverse experimental settings. For instance, the \texttt{MaxIterationEM=20} and \texttt{RelativeToleranceEM=1e-1}.

\section{Derivation of the Closed-Form Solution for $\mathbf{U}$}\label{sec:closed_form}
As defined in the paper, fuzzy clustering analysis aims to solve the following optimization problem: 
\begin{equation}
\min_{\mathbf{U},\mathbf{V}}\sum_{j=1}^C\sum_{i=1}^M(u_{ij})^r||\bm{x}_i-\bm{v}_j||_2^2, s.t.\sum_{j=1}^Cu_{ij}=1, u_{ij}\in[0,1].
\label{eq:fcm_origin}
\end{equation} We use Lagrange multiplier method to convert \cref{eq:fcm_origin} as a unconstrained optimization problem:
\begin{equation}
\mathcal{J}({\mathbf{U},\mathbf{V}}, \mathbf{\lambda})=\sum_{j=1}^C\sum_{i=1}^M(u_{ij})^r||\bm{x}_i-\bm{v}_j||_2^2+\sum_{i=1}^M\lambda_i(\sum_{j=1}^Cu_{ij}-1), 
\end{equation}where $\mathbf{\lambda}=[\lambda_1, \lambda_2, \cdots, \lambda_M]^{\top}$ are the set of Lagrange multipliers. Without loss of generality, we have 
\begin{equation}
\frac{\partial\mathcal{J}}{\partial u_{ij}}=r(u_{ij})^{r-1}||\bm{x}_i-\bm{v}_j||_2^2+\lambda_i=0. 
\end{equation} Therefore, we obtain 
\begin{equation}
 u_{ij}=\left(\frac{-\lambda_i}{r||\bm{x}_i-\bm{v}_j||_2^2}\right)^{\frac{1}{r-1}}.
 \label{eq:u_ij}
\end{equation} As $\sum_{j=1}^Cu_{ij}=1$, we have 
\begin{equation}
\sum_{j=1}^C\left(\frac{-\lambda_i}{r||\bm{x}_i\!-\!\bm{v}_j||_2^2}\right)^{\frac{1}{r\!-\!1}}\!=\!(\frac{-\lambda_i}{r})^{\frac{1}{r-1}}\sum_{j=1}^C\left(\frac{1}{||\bm{x}_i\!-\!\bm{v}_j||_2^2}\right)^{\frac{1}{r\!-\!1}}\!=\!1.
\end{equation}Thereby, we have 
\begin{equation}
 (\frac{-\lambda_i}{r})^{\frac{1}{r-1}}= \frac{1}{\sum_{j=1}^C\left(\frac{1}{||\bm{x}_i-\bm{v}_j||_2^2}\right)^{\frac{1}{r-1}}}. 
 \label{eq:lambda_i}
\end{equation} Substituting \cref{eq:lambda_i} into \cref{eq:u_ij}, we obtain the closed-form solution of $u_{ij}$ as
\begin{equation}
u_{ij}=\frac{1}{\sum_{k=1}^{C}\left(\frac{\left\|\boldsymbol{x}_i-\boldsymbol{v}_j\right\|^2_2}{\left\|\boldsymbol{x}_i-\boldsymbol{v}_k\right\|^2_2}\right)^{\frac{1}{r-1}}}. \bm{x}_5
\label{eq:U_origin}
\end{equation} This completes the derivation.

\section{Theoretical Proof}
\label{sec:1}
We present the theoretical proof that relates the pairwise distance computation of clustering centroids and samples with the orthogonal distance as follows.

\begin{proposition}
The sum of weighted pairwise distances from each clustering sample $\bm{x}_i\in \mathbb{R}^n$ to all clustering centroids $\mathbf{V}=\{\bm{v}_j\in\mathbb{R}^n\}$ in \begin{equation}
\min_{\mathbf{U},\mathbf{V}}\sum_{j=1}^C\sum_{i=1}^M(u_{ij})^r||\bm{x}_i-\bm{v}_j||_2^2, s.t.\sum_{j=1}^Cu_{ij}=1, u_{ij}\in[0,1].
\label{eq:fcm_origin}
\end{equation} is equivalent to the closest distance from $\bm{x}_i$ to $\mathbf{V}$ when the fuzzy factor $r\rightarrow1$, i.e., 
\begin{equation}
\lim_{r\rightarrow1}\left(\sum_{j=i}^C(u_{ij})^r||\bm{x}_i-\bm{v}_j||_2^2\right)= \min_{j=1,\cdots, C}||\bm{x}_i-\bm{v}_j||_2^2. 
\end{equation}

%their geometrically shortest distance. 
\label{proposition:1}
\end{proposition}

\begin{proof}
For simplicity, we insert
$$u_{ij}=\frac{1}{\sum_{k=1}^{C}\left(\frac{\left\|\boldsymbol{x}_i-\boldsymbol{v}_j\right\|^2_2}{\left\|\boldsymbol{x}_i-\boldsymbol{v}_k\right\|^2_2}\right)^{\frac{1}{r-1}}}$$ in \begin{equation}
\min_{\mathbf{U},\mathbf{V}}\sum_{j=1}^C\sum_{i=1}^M(u_{ij})^r||\bm{x}_i-\bm{v}_j||_2^2, 
\end{equation} to obtain
\begin{equation}
\sum_{j=1}^C\sum_{i=1}^M(\frac{(\left\|\boldsymbol{x}_i-\boldsymbol{v}_j\right\|^2_2)^{\frac{1}{1-r}}}{\sum_{k=1}^{C}\left({\left\|\boldsymbol{x}_i-\boldsymbol{v}_k\right\|^2_2}\right)^{\frac{1}{1-r}}})^r||\bm{x}_i-\bm{v}_j||_2^2 .
\label{eq:fcm_proof}
\end{equation} After an algebraic operation,   \cref{eq:fcm_proof} equals to
\begin{equation}
\sum_{i=1}^M(\sum_{j=1}^C(\left\|\boldsymbol{x}_i-\boldsymbol{v}_j\right\|^2_2)^{\frac{1}{1-r}})^{1-r}.
\label{eq:fcm}
\end{equation}
According to Lemmas 1 and 2 presented in the paper, we have
\begin{equation*}
\begin{aligned}
\min _{j \in\{1, \cdots, C\}} \frac{\left\|\boldsymbol{x}_i-\boldsymbol{v}_j\right\|^2_2}{C^{1 / v}} &\leq (\sum_{j=1}^C (\left\|\boldsymbol{x}_i-\boldsymbol{v}_j\right\|^2_2)^{-v})^{-1/v}\\ 
&\leq \min _{j \in\{1, \cdots, C\}}\left\|\boldsymbol{x}_i-\boldsymbol{v}_j\right\|^2_2,
\end{aligned}
\end{equation*} where $v=\frac{1}{r-1}.$ As $$\lim\limits_{r\rightarrow1}\min\limits_{j \in\{1, \cdots, C\}} \frac{\left\|\boldsymbol{x}_i-\boldsymbol{v}_j\right\|^2_2}{C^{1 / v}}=\min\limits_{j \in\{1, \cdots, C\}}\left\|\boldsymbol{x}_i-\boldsymbol{v}_j\right\|^2_2,$$
according to the squeeze theorem, we have 
\begin{equation*}
\lim_{r\rightarrow1}(\sum_{j=1}^C (\left\|\boldsymbol{x}_i-\boldsymbol{v}_j\right\|^2_2)^{-v})^{-1/v} =\min _{j \in\{1, \cdots, C\}}\left\|\boldsymbol{x}_i-\boldsymbol{v}_j\right\|^2_2.
\end{equation*}   
This completes the proof.
\end{proof}

\section{Definition of the Spatial Inverse Mapping}\label{sec:transformation}
As we assess the deviation in the recovery of a superquadric in canonical form, we provide a comprehensive definition or derivation of the spatial inverse mapping. This mapping encompasses the transformations, including rigid, tapering, and bending, that revert a superquadric, initially situated in a general spatial location, back to its canonical form. The detailed formulation is as follows:

% Assuming you will then provide the equations and derivations after this introduction.

%As we calculate the superquadric recovery deviation in the canonical form, we present the detailed definition or derivation of the spatial inverse mapping of the superquadric transformations including rigid, tapering, and bending that transform a superquadric situated at a general spatial location in a canonical form.
\begin{figure*}[t]
       \centering
       \subcaptionbox{Input}{\includegraphics[width=0.19\linewidth]{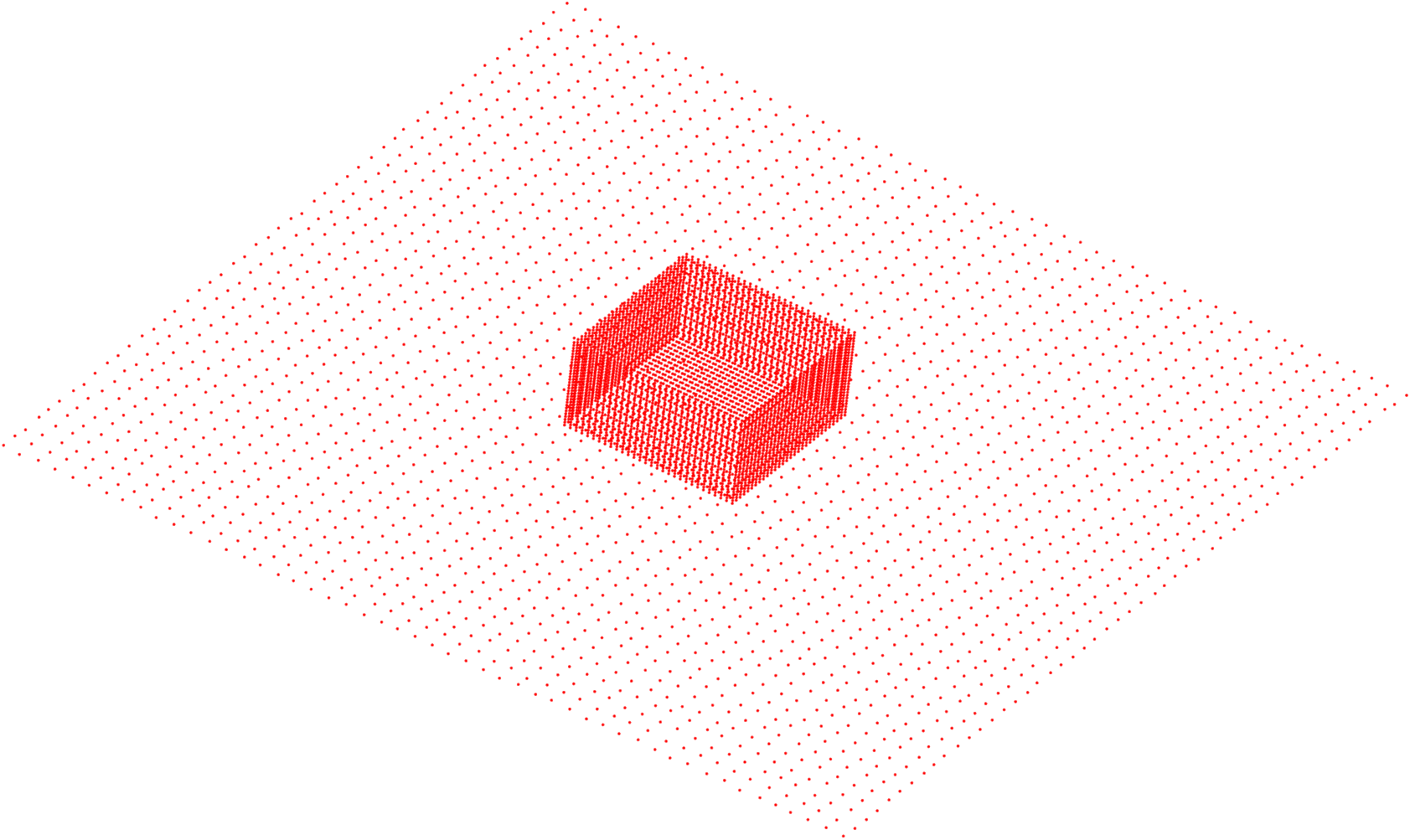}}
       \subcaptionbox{Output}{\includegraphics[width=0.19\linewidth]{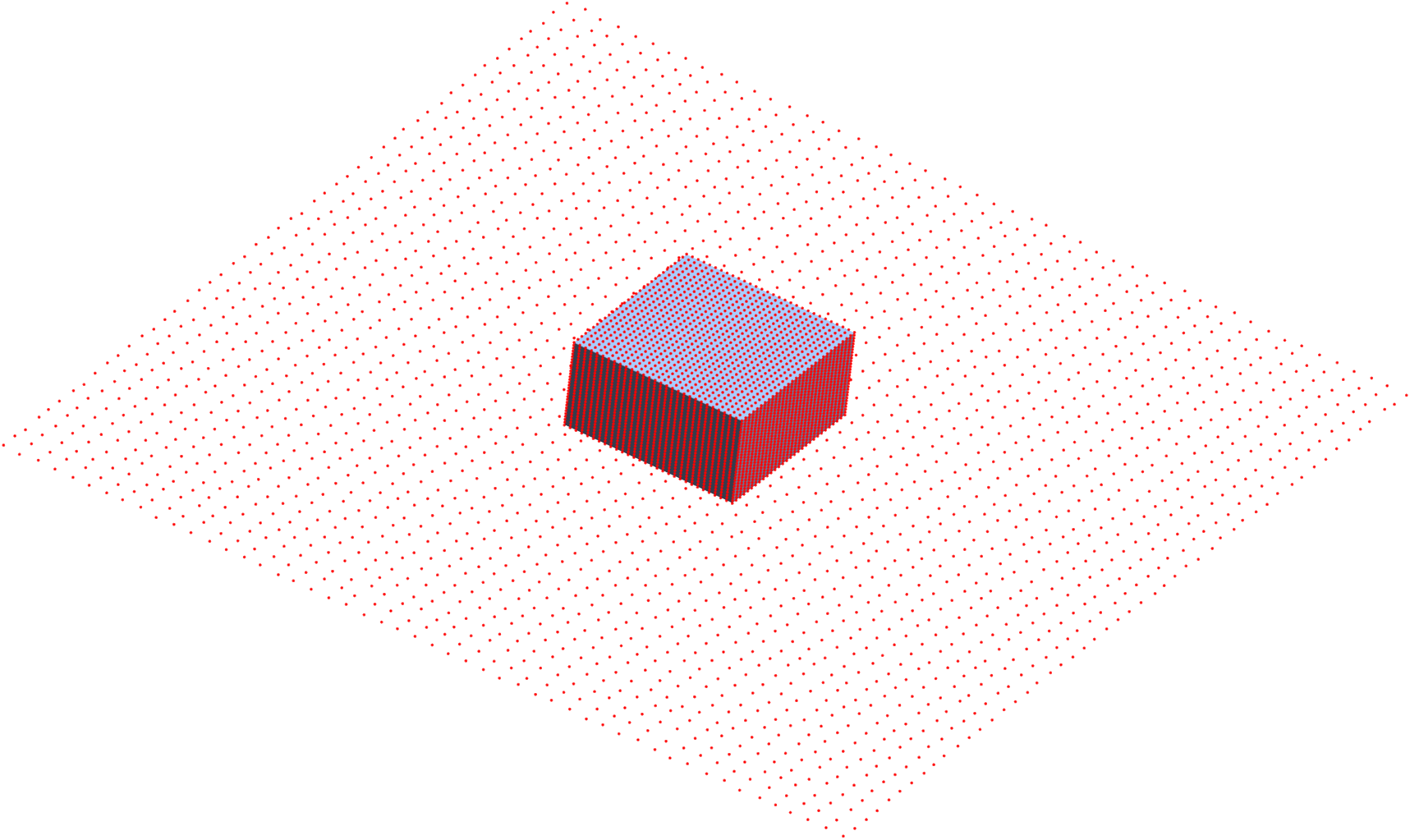}}
       \subcaptionbox{Input}{\includegraphics[width=0.19\linewidth]{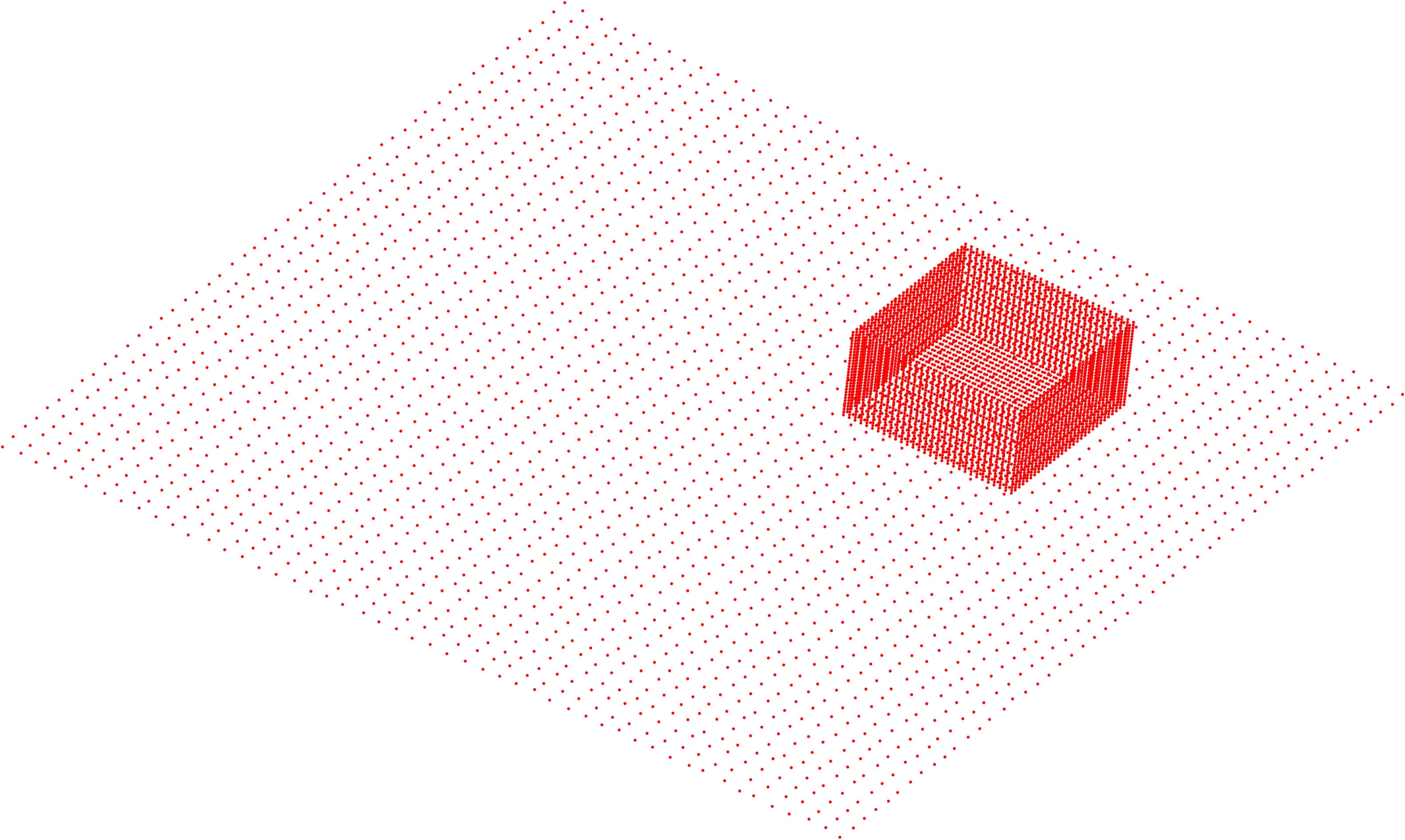}}
       \subcaptionbox{Output 1}{\includegraphics[width=0.19\linewidth]{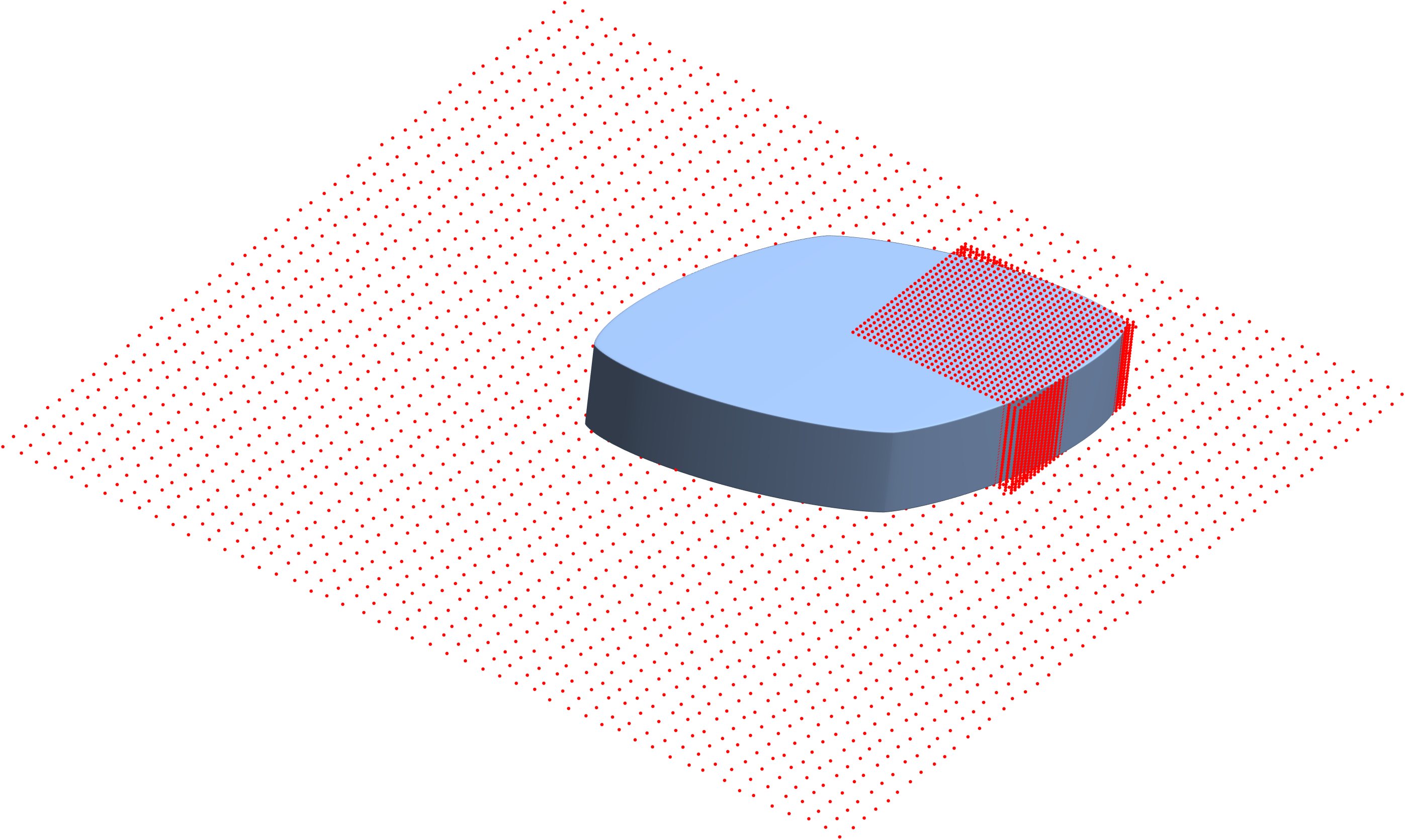}}
       \subcaptionbox{Output 2}{\includegraphics[width=0.19\linewidth]{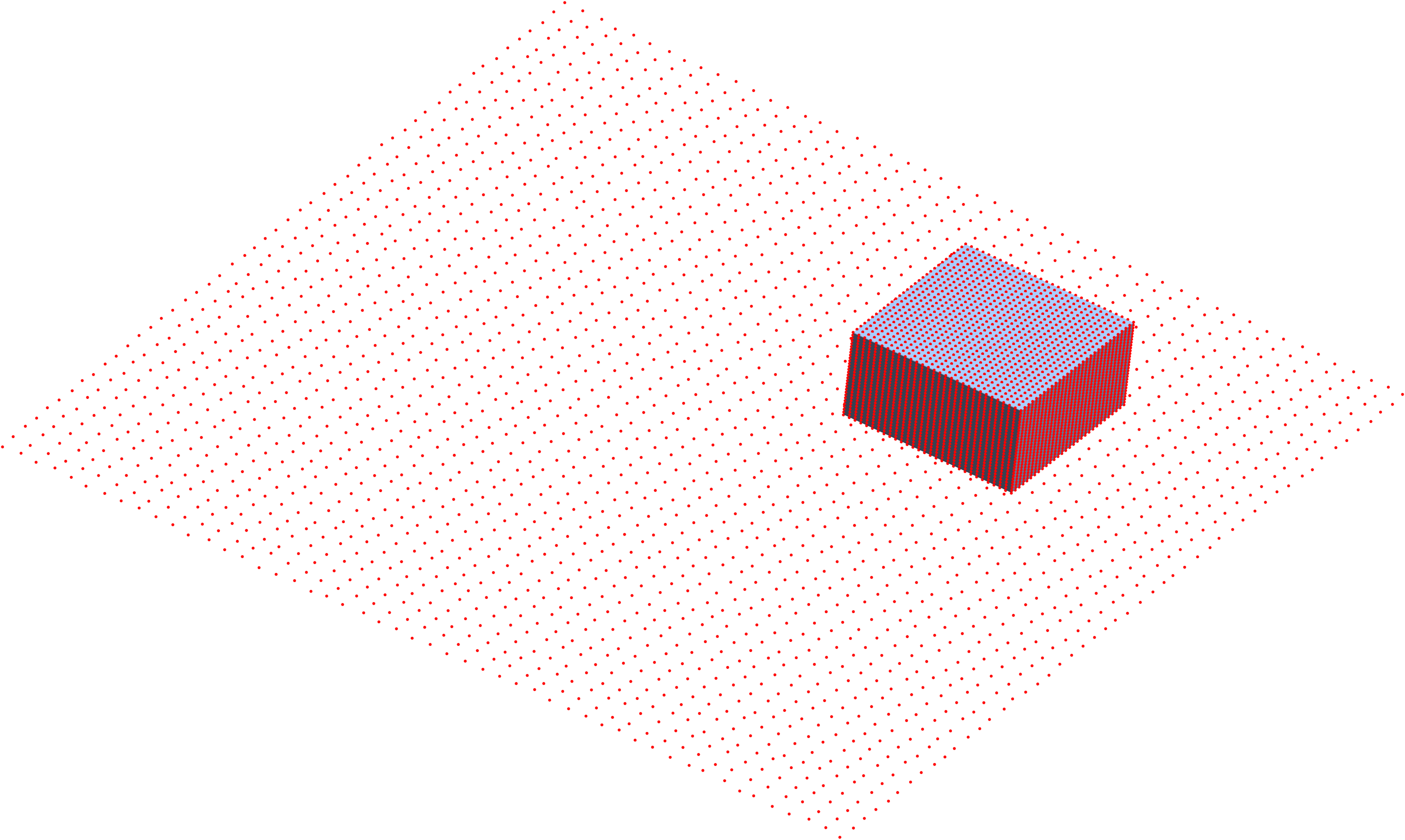}}
       \caption{\revise{Investigation of the proposed method against structured noise.}}
       \label{fig:structure}
   \end{figure*}
\subsection*{Inverse Rigid Transformation}
Given a point $[x', y', z']^{\top}\in\mathbb{R}^3$ situated at a general spatial location, its inverse rigid transformation is formally defined as
\begin{equation}
[x', y', z']^{\top}=\widehat{\mathbf{R}}^{\top}([x', y', z']^{\top}-\widehat{\mathbf{t}}).   
\end{equation} Here, $\widehat{\mathbf{R}}\in\mathbb{R}^{3\times3}$ is the estimated rotation matrix and  $\widehat{\mathbf{t}}\in\mathbb{R}^{3}$ is the estimated translation vector.

\subsection{Inverse Tapering Transformation} 
After the superquadric size parameters $a_x, a_y, a_z$ and the tapering degrees $k_{x}, k_{y}$ along the $x$ and $y$ axes have been recovered, the inverse tapering transformation can be formally expressed as follows:
\begin{equation}
\left\{
\begin{aligned}
 x&=\frac{x'}{f_{k_{x}}(z)}=\frac{x'}{\frac{{k_{x}}}{a_z}z+1}\\
y&=\frac{y'}{f_{k_{y}}(z)}=\frac{y'}{\frac{{k_{y}}}{a_z}z+1}\\  
z&=z'
\end{aligned}.
\right.
\end{equation} Here, $a_z$ represents the estimated superquadric size along the $z$ axis. The point that $[x', y', z']^{\top}$ undergoes the rigid transformation is still denoted as   
$[x', y', z']^{\top}=\widehat{\mathbf{R}}^{\top}([x', y', z']^{\top}-\widehat{\mathbf{t}})$, where $\widehat{\mathbf{R}}\in\mathbb{R}^{3\times3}$ is the rotation matrix and  $\widehat{\mathbf{t}}\in\mathbb{R}^{3}$ is the translation vector. %denoting the inverse rigid transformation. 

\subsection{Inverse Bending Transformation} After recovering the bending degree $\kappa\in\mathbb{R}_{+}$ and the direction $\alpha$ in the \(x\)-\(y\) plane, the inverse bending transformation for a point  $[x', y', z']^{\top}\in\mathbb{R}^3$ situated in a general spatial position can be formally expressed as 
\begin{equation}
\left\{
\begin{aligned}
& x={x}'-\cos (\alpha)(R-r) \\
& y={y}'-\sin (\alpha)(R-r) \\
& z=\kappa^{-1} \gamma
\end{aligned},
\right.
\end{equation}where 
\begin{equation}
\left\{
\begin{aligned}
\gamma & =\arctan \frac{{z}'}{\kappa^{-1}-R} \\
r & =\kappa^{-1}-\sqrt{{z}'^2+\left(\kappa^{-1}-R\right)^2} \\
R & =\cos \left(\alpha-\arctan\frac{{y}'}{{x}'}\right) \sqrt{{x}'^2+{y}'^2}
\end{aligned}
\right.
\end{equation}
and $[x', y', z']^{\top}=\widehat{\mathbf{R}}^{\top}([x', y', z']^{\top}-\widehat{\mathbf{t}})$ with $\widehat{\mathbf{R}}\in\mathbb{R}^{3\times3}$ and  $\widehat{\mathbf{t}}\in\mathbb{R}^{3}$ denoting the inverse rigid transformation.

\section{Optimization Bound Setting}
\label{sec:2}
During the optimization of the unknown superquadric parameters $\mathbf{\Theta}$, we set different upper and lower bounds (abbreviated as ub and lb) for each parameter to boost the convergence of the interior trust region method. Specifically, \\\\
(1) for rigid superquadrics, we have \\
\texttt{lb = [0.0 0.0 0.00001 0.00001 0.00001 -2*pi -2*pi -2*pi -ones(1, 3) * upper],}\\
\texttt{
ub = [2.0 2.0 ones(1, 3) * upper  2*pi 2*pi 2*pi ones(1, 3) * upper].}\\\\
(2) For tapering deformations, since the tapering degree $k_1, k_2\in[-1, 1]$, we have \\
\texttt{lb = [0.0 0.0 0.001 0.001 0.001 -2*pi -2*pi -2*pi -ones(1, 3) * upper,-1,-1],}\\
\texttt{ub = [2.0 2.0 ones(1, 3) * upper  2*pi 2*pi 2*pi ones(1, 3) * upper,1,1].}\\\\ 
(3) For bending deformations, since $\kappa=0$ will make Eq.~(6) in the paper singular, we set the bounds as\\
\texttt{lb = [0.0 0.0 0.00001 0.00001 0.00001 -2*pi -2*pi -2*pi -ones(1, 3) * upper, 0.0001 0],}\\
\texttt{ub = [2.0 2.0 ones(1, 3) * upper  2*pi 2*pi 2*pi ones(1, 3) * upper, pi/2]}.\\
Here, \texttt{upper} is defined as\\
\texttt{upper = 4 * max(max(abs(point))).}\\
In contrast to previous approaches that constrain the region of the shape parameters $\varepsilon_1$ and $\varepsilon_1$, our optimization bounds incorporate all convex superquadrics, as our method is inherently stable.

{\section{Robustness against Structured Noise}
We conduct additional experiments to evaluate the robustness of the proposed method against structured noise. As shown in Fig.~\ref{fig:structure}, we construct a scene consisting of a floor plane beneath a box. In Fig.~\ref{fig:structure}(a), the box is located at the center of the plane, and our method achieves high-quality fitting. In Fig.~\ref{fig:structure}(c), the box is placed far from the center. Due to the resulting deviation in the initialization of the center and axes, our method produces fitting results that incorporate both the floor and the box point clouds (Fig.~\ref{fig:structure}(d)). Fig.~\ref{fig:structure}(e) shows a successful fitting result when the center is initialized near the box. These experiments confirm that structured noise indeed affects the fitting quality, primarily by biasing the initialization. To ensure high-quality fitting, we recommend a coarse segmentation step to obtain a better initialization.}

% \bibliographystyle{ieeetr}
% \bibliography{main}
\end{document}